\documentclass
[dvipsnames,
    sigconf = true,     % Style for GECCO proceedings (two columns and such).
    review = false,      % Shows line numbers.
    screen = true,      % Highlights links.
    anonymous = false,   % Anonymous author information. (Also removes the Acknowledgment section.)
]
{acmart}

\newboolean{arxiv} %Deklaration
\setboolean{arxiv}{true} %Zuweisung

\allowdisplaybreaks[4]

\usepackage[utf8]{inputenc}
\usepackage{xspace}
\usepackage{amsmath,amsthm,mathtools}
\usepackage{bbm}
\usepackage{url}
\usepackage[caption=false]{subfig}
\usepackage{algorithm}
\usepackage[noend]{algpseudocode}
\usepackage{xcolor}
\usepackage{tikz}
\usepackage{dsfont}
\usepackage{graphicx}
\usepackage{enumitem}

\newcommand{\prob}[1]{\mathord{\Pr}\mathord{\left[#1\right]}}

\newcommand{\expect}[1]{\mathord{E}\mathord{\left[#1\right]}}

\newcommand{\ie}{i.\,e.\xspace}
\newcommand{\eg}{e.\,g.\xspace}
\newcommand{\wrt}{w.\,r.\,t.\xspace}

\newcommand{\om}{\textsc{OneMax}\xspace}
\newcommand{\onemax}{\om}
\newcommand{\cliff}{\textsc{Cliff}\xspace}

\newcommand{\valley}{\textsc{Valley}\xspace}
\newcommand{\valleypath}{\textsc{ValleyPath}\xspace}
\newcommand{\jump}{\textsc{Jump}\xspace}
\newcommand{\DLB}{\textsc{DLB}\xspace}
\newcommand{\oea}{\mbox{${(1 + 1)}$~EA}\xspace}

\newcommand{\ooea}{\oea}

\newcommand{\mplea}{\mbox{${(\mu+\lambda)}$~EA}\xspace}

\newcommand{\oclea}{\mbox{${(1,\lambda)}$~EA}\xspace}

\newcommand{\R}{\ensuremath{\mathbb{R}}}

\newcommand{\N}{\ensuremath{\mathbb{N}}} 

\newcommand{\bbone}{{\mathds{1}}} % not used in the document
\newcommand{\ones}[1]{\| #1\|_1}

\newcommand{\ceil}[1]{\lceil #1\rceil}
\newcommand{\floor}[1]{\lfloor #1\rfloor}

\newenvironment{proofof}[1]{\begin{proof}[Proof of~#1]}{\end{proof}}

\newcommand{\eps}{\varepsilon} 
\newcommand{\ymin}{y_{\min}} 
\copyrightyear{2023}
\acmYear{2023}
\setcopyright{acmlicensed}\acmConference[GECCO '23]{Genetic and
Evolutionary Computation Conference}{July 15--19, 2023}{Lisbon, Portugal}
\acmBooktitle{Genetic and Evolutionary Computation Conference (GECCO '23),
July 15--19, 2023, Lisbon, Portugal}
\acmPrice{15.00}
\acmDOI{10.1145/3583131.3590390}
\acmISBN{979-8-4007-0119-1/23/07}

\begin{document}

\title[How Well Does the Metropolis Algorithm Cope With Local Optima?]{How Well Does the Metropolis Algorithm Cope\\ With Local Optima?}\ifthenelse{\boolean{arxiv}}{\titlenote{Author-generated version.}}{}

\author{Benjamin Doerr}
\affiliation{
\institution{Laboratoire d'Informatique (LIX)\\
CNRS, École Polytechnique\\ Institut Polytechnique de Paris}
\city{Palaiseau}
\country{France}
}

\author{Taha El Ghazi El Houssaini}
\affiliation{
\institution{École Polytechnique\\ Institut Polytechnique de Paris}
\city{Palaiseau}
\country{France}
}

\author{Amirhossein Rajabi}
\affiliation{
\institution{DTU Compute\\
Technical University of Denmark}
\city{Kgs. Lyngby}
\country{Denmark}
}
\author{Carsten Witt}
\affiliation{
\institution{DTU Compute\\
Technical University of Denmark}
\city{Kgs. Lyngby}
\country{Denmark}
}


{\sloppy
\begin{abstract}
 The Metropolis algorithm (MA) is a classic stochastic local search heuristic. It avoids getting stuck in local optima by occasionally accepting inferior solutions. To better and in a rigorous manner understand this ability, we conduct a mathematical runtime analysis of the MA on the CLIFF benchmark. Apart from one local optimum, cliff functions are monotonically increasing towards the global optimum. Consequently, to optimize a cliff function, the MA only once needs to accept an inferior solution. Despite seemingly being an ideal benchmark for the MA to profit from its main working principle, our mathematical runtime analysis shows that this hope does not come true. Even with the optimal temperature (the only parameter of the MA), the MA optimizes most cliff functions less efficiently than simple elitist evolutionary algorithms (EAs), which can only leave the local optimum by generating a superior solution possibly far away. This result suggests that our understanding of why the MA is often very successful in practice is not yet complete. Our work also  suggests to equip the MA with global mutation operators, an idea supported by our preliminary experiments. 
\end{abstract}

% <concept_id>10003752.10010070.10011796</concept_id>
% <concept_desc>Theory of computation~Theory of randomized search heuristics</concept_desc>

\ccsdesc[500]{Theory of computation~Theory of randomized search heuristics}

\keywords{Metropolis algorithm, stochastic local search heuristic, evolutionary algorithm, runtime analysis, theory.}

\maketitle
\ifthenelse{\boolean{arxiv}}{\pagestyle{plain}}{}
\section{Introduction}
A major difficulty faced by many search heuristics is that the heuristic might run into a local optimum and then find it hard to escape from it. A number of mechanisms have been proposed to overcome this difficulty, \eg, restart mechanisms, discarding good solutions (non-elitism), tabu mechanisms, global mutation operators (which can, in principle, generate any solution as offspring), or diversity mechanisms (which prevent a larger population to fully converge into a local optimum). While all these ideas have been successfully used in practice, a rigorous understanding of how these mechanisms work and in which situation to employ which one, is still largely missing.

To shed some light on this important question, we analyze how the Metropolis algorithm (MA) profits from its mechanism to leave local optima. The MA is a simple randomized hillclimber except that it can also accept an inferior solution. This happens with some small probability which depends on the degree of inferiority and the \emph{temperature}, the only parameter of the MA. Choosing the right temperature is a delicate problem -- a too low temperature makes it hard to leave local optima, whereas a too high temperature forbids an effective hillclimbing.

From this description of the MA one might speculate that the MA copes particularly well with local optima that are close (and thus easy to reach) to inferior solutions from which improving paths lead away from the local optimum. The main result of this work is that this is not true. We conduct a rigorous runtime analysis of the MA on the \cliff benchmark, in which every local optimum is a neighbor of a solution from which improving paths lead right to the global optimum. For this classic benchmark, we prove that the MA even with the optimal (instance-specific) temperature is less efficient on most problem instances than a simple elitist mutation-based algorithm called \oea with standard mutation rate. If the \oea uses an optimized mutation rate, then this discrepancy is even more pronounced. Our experimental results support these findings and show that also several other simple heuristics using global mutation clearly outperform the MA on cliff functions. These results have motivated us to conduct preliminary experiments with the MA equipped with a global mutation operator instead of the usual one-bit flips. While not fully conclusive, these experiments generally show a good performance of the MA with global mutation operators on \cliff. We note that this idea, replacing local mutation by global mutation in a local search heuristic was taken up in~\cite{DoerrDLS23} and proven to give significant performance gains when the move acceptance hyper-heuristic optimizes the \cliff benchmark.
%This paper is structured as follows. After a description of the most relevant previous works in Section~\ref{sec:previous}, we define in Section~\ref{sec:preliminaries} the algorithms and benchmark problems under consideration. The heart of this work is Section~\ref{sec:runtime}, which is devoted to the mathematical runtime analysis and the mathematical performance comparison of the MA and \oea. Section~\ref{sec:experiments} presents experimental supplements to the theoretical analysis. The brief experimental section emphasizes that the performance difference found in the asymptotic mathematical analysis is clearly visible for small problem sizes already. 

\ifthenelse{\boolean{arxiv}}
{For reasons of space, in the conference version~\cite{DoerrERW23} we can only sketch the proofs of our mathematical results. The appendix of this preprint contains the missing proofs.}
{For reasons of space, some proofs had to be omitted in this extended abstract. They can be found in the preprint~\cite{DoerrERW23arxiv}.}

\section{Previous Works}
\label{sec:previous}

The mathematical runtime analysis of randomized search heuristics has produced a decent number of results on how elitist evolutionary algorithm cope with local optima, but much fewer on other algorithms. The majority of results on evolutionary algorithms concern mutation-based algorithms. Results derived from the \jump benchmark suggest that higher mutation rates or a heavy-tailed random mutation rate~\cite{DoerrLMN17} as well as a stagnation-detection mechanism~\cite{RajabiW22,RajabiW23,RajabiW21gecco,DoerrR23} can speed up leaving local optima. Some examples have been given where elitist crossover-based algorithms coped remarkably well with local optima~\cite{JansenW02,DangFKKLOSS18,RoweA19,AntipovDK20}, but it is not clear to what extent these results generalize~\cite{Witt21}.

There are a few runtime results on non-elitist evolutionary algorithms, however, they do not give a very conclusive picture. The results of \citet{JagerskupperS07, Lehre10, Lehre11, RoweS14, Doerr22, FajardoS21foga} show that in many situations, there is essentially no room between a regime with low selection pressure, in which the algorithm cannot optimize any function with unique optimum efficiently, and a regime with high selection pressure, in which the algorithm essentially behaves like its elitist counterpart. Only with a very careful parameter choice, one can profit from non-elitism in a small middle regime. For example, with a population size of 
order $\Theta(\log n)$, 
%of $\lambda \ge 5 \ln n$, the 
the $\oclea$ can optimize the function $\cliff_{\frac n3-\frac 32,\frac n3}$ in polynomial time \cite{FajardoS21foga}.
%
%time $O(\exp(5\lambda)) \ge n^{25}$. Via a stronger analysis, %\cite{FajardoS21} brought the runtime down to $O(n^{3.98})$ %approximately.
However, the exponential dependence of the runtime on $\lambda$, roughly $6.20^\lambda$, implies that this algorithm parameter has to be chosen very carefully. Other examples of successful applications of non-elitism in evolutionary algorithms exist, e.g.,~\citet{DangEL21aaai}. Most of these works do not regard classic benchmarks, but artificial problems designed to demonstrate that a particular behavior can happen, so it is usually difficult to estimate how widespread this behavior really is. The very recent work~\cite{JorritsmaLS23} shows moderate advantages of comma selection on randomly disturbed \onemax functions.

Outside the range of well-established search heuristics, \citet{PaixaoHST17} show that the strong-selection weak-mutation process from biology can optimize some functions faster than elitist evolutionary algorithms. \citet{LissovoiOW23} show that the move-acceptance hyper-heuristic proposed by \citet{LehreO13} can optimize cliff functions in cubic time. However, as recently shown in~\cite{DoerrDLS23}, it performs significantly worse than most EAs on the \jump benchmark.

For the MA algorithm, the rigorous understanding is less developed than for EAs. The classic result of \citet{SasakiH88} shows that the MA can compute good approximations for the maximum matching problem. An analogous result was shown for the \oea~\cite{GielW03}, demonstrating that this problem can also be solved via elitist methods. \citet{JerrumS98} showed that the MA can solve certain random instances of the minimum bisection problem in quadratic time. 

\citet{JansenW07} conducted a runtime analysis of~MA on the classic \onemax benchmark. While it is not surprising that the MA does not profit from its ability to accept inferior solutions on this unimodal benchmark, their result shows that only very small temperatures (namely such that the probability of accepting an inferior solution is at most $O(\log(n) / n)$) lead to polynomial runtimes. As a side result to their study on hyperheuristics, \citet[Theorem~14]{LissovoiOW23} show that the MA cannot optimize the multimodal \jump benchmark in sub-exponential time. The same work also contains a runtime analysis on the \cliff problem, which we will discuss in more detail after having introduced this benchmark further below. \citet{WangZD21} show a good performance of the MA on the \DLB benchmark (roughly by a factor of $n$ faster than elitist EAs). This problem, first proposed by \citet{LehreN19foga} has (many) local optima, however, these are easy to leave since they all have a strictly better solution in Hamming distance two.

Again a number of results exist for artificially designed problems. Among them, \citet{DrosteJW00} defined a \valley problem that is hard to solve for the MA with any fixed temperature, whereas Simulated Annealing, that is, the MA with a suitable cooling schedule, solves it in polynomial time. \citet{JansenW07} construct an objective function such that the MA with a very small temperature has a polynomial runtime, whereas the \oea needs time $\Omega(n^{\Omega(\log \log n)})$. \citet{OlivetoPHST18} proposed a problem called \valley, different from the homonymous \valley problem defined by \citet{DrosteJW00}, such that again the MA and the Strong Selection Weak Mutation algorithm have a much better runtime than the \oea. Similar results where obtained for a similar problem called \valleypath, which contains a chain of several local optima. It should be noted that all these problems were designed to demonstrate a particular difference between two algorithms and are even farther from real-world problems than the classic benchmarks like \onemax, \jump and \cliff. For example, the \valley and \valleypath problems defined by \citet{OlivetoPHST18} are essentially a one-dimensional problems that are encoded into the discrete hypercube. For this reason, it is hard to derive general insights from these works beyond the fact that the algorithms regarded can have drastically different performances.

\section{Preliminaries}
\label{sec:preliminaries}

\subsection{The Metropolis Algorithm and the \oea}

The Metropolis Algorithm (MA)~\cite{MetropolisRRTT53} is a simple single-trajectory search heuristic
for pseudo-Boolean optimization. It selects and evaluates a random neighbor of the current solution and accepts it (i)~always if it is at least as good as the parent, and (ii)~with probability $e^{-\delta/T}$ if its fitness is by $\delta$ worse than the fitness of the current solution. Here $T$, often called \emph{temperature}, is the single parameter of the MA. See Algorithm~\ref{alg:metropolis} for the pseudocode of the MA. To ease our later analyses, we use the parameterization $\alpha = e^{1/T}$, that is, the parameter $\alpha>0$ fixes the probability $\alpha^{-\delta}$ of accepting a solution worse than the parent by $\delta$. The MA and its 
generalization {Simulated Annealing} have found numerous successful applications in various areas, see, e.g.,~\citet{LaarhovenA87,DowslandT12}.

\begin{algorithm}[th]
	\caption{The Metropolis algorithm with temperature $T$ for the maximization of $f\colon\{0,1\}^n\to \R$. We ususally write $\alpha = e^{1/T}$.}
	\label{alg:metropolis}
	\begin{algorithmic}
		\State Select $x^{(0)}$ uniformly at random from $\{0, 1\}^n$.
		\For{$t \gets 0, 1, \dots$}
		\State Create $y$ by flipping a bit of~$x^{(t)}$ chosen uniformly at random.
		\If{$f(y) \ge f(x^{(t)})$}
		\State $x^{(t+1)} \gets y$
		\Else{}
		\State $x^{(t+1)} \gets y$ with probability $e^{(f(y)-f(x^{(t)}))/T}$ and 
		\State $x^{(t+1)} \gets x^{(t)}$ with the remaining probability.
		\EndIf
		\EndFor
	\end{algorithmic}
\end{algorithm}

To understand to what extent the MA profits from its ability to accept inferior solutions, we compare it with a simple \emph{elitist} stochastic hillclimber, the \oea (Algorithm~\ref{alg:oea}). As the MA, it follows a single search trajectory, however, it never accepts inferior search points. To be able to leave local optima, the \oea does not move to random neighbors of the current solution (that is, generates the new solution by flipping a random bit), but flips each bit of the current solution independently with some probability~$p$.  A common recommendation for this \emph{mutation rate} is $p = 1/n$, see, e.g., \citet{Back96,DrosteJW02,Ochoa02,Witt13}. With this choice, with probability approximately $1/e \approx 0.37$ the new solution is a neighbor of the current search point.
%The \oea is intensively studied in the theory of evolutionary computation \cite{DrosteJW02} and has, despite its simplicity, led to many  and serves as the basis for the study of more advanced evolutionary algorithms. In our formulation, it comes with the parameter~$p$ for the mutation rate.
 
\begin{algorithm}[th]
\caption{The \oea with mutation rate~$p$ for the maximization 
of $f\colon\{0,1\}^n\to \R$.}
\label{alg:oea}
\begin{algorithmic}
	\State Select $x^{(0)}$ uniformly at random from $\{0, 1\}^n$.
		\For{$t \gets 0, 1, \dots$}
		\State Create $y$ by flipping each bit of~$x^{(t)}$ 
		independently with probability~$p$.
		\If{$f(y) \ge f(x^{(t)})$} 
		 $x^{(t+1)} \gets y$ \Else{} $x^{(t+1)} \gets x^{(t)}$.
		 \EndIf
		\EndFor
		\end{algorithmic}
\end{algorithm}

As  \emph{runtime}~$T$ (synonymously, \emph{optimization time}) of these algorithms, we regard the (random) first point in time~$t$ where an optimum has been sampled. We shall mostly be interested in expected runtimes.

\subsection{The Cliff and OneMax Functions}

The aim of this paper is to study how efficient the MA is at optimizing functions with a local optimum.
The two best-studied benchmark functions to model situations with local optima are \jump~\cite{DrosteJW02} and \cliff~\cite{JagerskupperS07}. That the MA has enormous difficulties optimizing \jump was shown by \citet{LissovoiOW23}. This result is not too surprising when considering the fitness landscape of \jump. The valley of low fitness separating the local optimum from the global one is deep (it consists of the solutions of lowest fitness) and the fitness gradient is pointing towards the local optimum everywhere in this valley. 

For this reason, in this work we analyze the performance of the MA on the \cliff benchmark, where the valley of low fitness is more shallow and the fitness inside the valley is not deceptive, that is, the gradient is pointing towards the optimum. With these properties, \cliff should be a problem where the MA could profit from its ability to occasionally accept an inferior solution. Surprisingly, as our precise analysis for the full spectrum of temperatures will show, this is not true.

Like \jump functions, also \cliff functions were originally defined with only one parameter determining the distance of the local optimum from the global optimum, which also is the width of the valley of fitness lower than the one of the local optimum. 
Since a series of recent work on the \jump benchmark~\cite{Jansen15, BamburyBD21, RajabiW21gecco, DoerrZ21aaai, FriedrichKKR22, DoerrQ23tec, DoerrQ23crossover, DoerrQ23LB, Witt23, BianZLQ23} has shown that this restricted class of \jump functions can give misleading insights, we follow their example and extend also the \cliff benchmark to have two independent parameters for the distance between local and global optimum and the width of the valley of low fitness around the local optimum. This leads to the following definition of the \cliff benchmark.

Let $n \in \N$ denote the problem size, that is, the length of the bit-string encoding of the problem. As common, we shall usually suppress this parameter from our notation. 
Let $m\in\N_{\ge 1}$ and $d\in \R_{>0}$ such that $m<n$ and $d<m-1$. Then  we define
\[
\cliff_{d,m}(x) \coloneqq
%\biggl\{
    \begin{cases}
        \ones{x}  & \mbox{if } \ones{x} \leq n- m,\\
        \ones{x} - d -1 & \mbox{otherwise}
    \end{cases}
%\biggr.
\]
for all $x \in \{0,1\}^n$, where $\|x\|_1$ is the number of one-bits in the bit string. We note that the  function $\cliff_{d,m}$ is increasing as the number of one-bits of the argument increases except for the points with~$n-m$ one-bits, where the fitness decreases sharply by~$d$  if we add one more one-bit to the search point. See Figure~\ref{fig:cliff} for an illustration. We note that the original cliff benchmark is the special case with $d=m-3/2$.

\begin{figure}[ht!]
    \centering
    \!\!\!\includegraphics[width=1.02\linewidth]{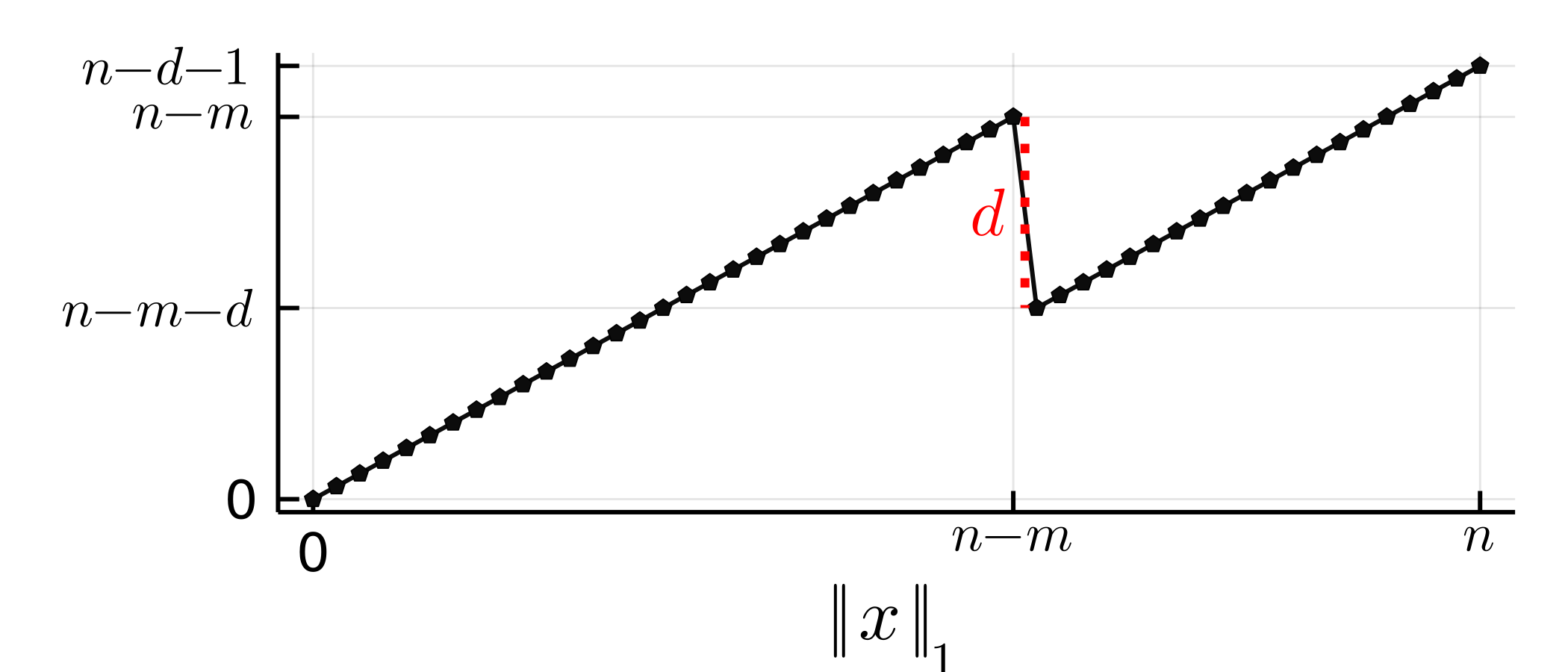}
    \caption{The function $\cliff_{d,m}$.}
    \label{fig:cliff}
\end{figure}

To the best of our knowledge, the only runtime analysis of the MA on \cliff functions was conducted by~\cite[Theorem~11]{LissovoiOW23}. On the original $\cliff_{d,m}$ function with fixed~$d=m-3/2$, the authors showed a lower bound for the runtime of
\[\min\left\{\frac{1}{2}\cdot \frac{n-m+1}{m-1}\cdot \left(cn/ \log n\right)^{m-3/2}, n^{\omega(1)}\right\},\]
where $c>0$ is a suitable constant. 
%As discussed above, this result is limited to \cliff functions where the valley of low fitness extends up to the global optimum, so it missed the larger diversity of the generalized \cliff benchmark. It 
For super-constant valley widths (and thus cliff heights), 
this result assures that the runtime of the Metropolis algorithm is super-polynomial. For constant~$m$, the $\tilde \Omega(n^{m-0.5})$ lower bound is  mildly lower than the $\Theta(n^m)$ runtime of the \oea. So this result leaves some room for a possible advantage from accepting inferior solutions for the restricted case that the valley extends to the global optimum.
% , and it does not rule out larger advantages on \cliff functions of different shape, that is, where $d \neq m-3/2$ (and we recall that for \jump functions, the shape had a significant influence on this algorithms had the best performance \cite{BamburyBD21}). However, since 
% the result by \citet{LissovoiOW23} is limited 
% to a lower bound for MA, no final conclusion 
% can be drawn from it.

On the two slopes of \cliff, the function is essentially equivalent to the classic \onemax benchmark
\[\onemax(x)  \coloneqq \ones{x}.\]
Consequently, understanding the MA on \onemax will be crucial for our analysis on \cliff. Intuitively, if the MA is 
run on \cliff, assuming $m\ll n/2$, it first of all has to optimize a \onemax-like function to 
reach the cliff, accept the drop to jump down the cliff and then again optimize 
a \onemax-like function to reach the global optimum. However, it may (and usually will) happen 
that the MA returns to the cliff point or even points left of the cliff again after having 
overcome it for the first time.

%for some $\alpha\ge 1$. 

%\merk{what does this mean? OK, but we can't write this like this even if LOW write this.}
%In this paper, we are interested in more rigorous bounds for the optimization time and a better understanding of the algorithm's behavior when confronted with local optima. \merk{short description of our main result, not too technical}

%To this end, we consider an extension of the classical \cliff functions defined as follows:

%Our definition extends the original \cliff function, first introduced by \cite{JaegerskuepperFOCI07}, 
%by the two parameters~$m$ and $d$

% \section{Mathematical Tools}

% \begin{lemma}[Wald's inequality from~\cite{DoerrK15r}]\label{lem:wald}
% Let $T$ be a random variable with a finite expectation, and let $X_1$, $X_2$, \dots be non negative random variables with $\expect{X_i\mid T\ge i}\le C$. Then
% \[\expect{\sum_{i=1}^{T}X_i }\le \expect{T}\cdot C.\]
% \end{lemma}

\section{Mathematical Runtime Analysis}
\label{sec:runtime}

\subsection{OneMax}
\label{sec:onemax}

As explained above, we start our mathematical analysis with a runtime analysis of the MA on \onemax. This results will be needed in our analysis on \cliff, but it is also interesting in its own right as it very precisely describes the transitions between the different parameter regimes. 

Many randomized search heuristics optimize the \onemax function in time $\Theta(n \log n)$. More precisely, a runtime of $(1 \pm o(1)) n \ln(n)$ is the best performance a heuristic can have which generates offspring from single previous solutions in an unbiased manner \cite{DoerrDY20}. It is easy to see that with a sufficiently small temperature (that is, $\alpha$ sufficiently large), the MA attains this runtime as well. Since in our analysis on \cliff smaller values of~$\alpha$ will be necessary to leave the local optimum, we need a runtime analysis also for such parameter values. 
% efficient and super-polynomial parameter
% Typical randomized search heuristics optimize the \onemax function in expected time $\Theta(n\log n)$ \cite{DrosteJW02,DoerrDY20}. It is intuitively clear that 
% the MA cannot profit from accepting inferior solutions on this simple function, and one can show by a simple 
% domination argument (already shown in \cite{JansenW07}) that increasing the parameter $\alpha$ to a point where the algorithm 
%  essentially becomes elitist is the best choice. However, it is important to understand how robust the MA is 
%  \wrt the parameter $\alpha$ and what degree of non-elitism still allows efficient runtimes. In particular, 
%  we will need this analysis when we turn to the multimodal \cliff function in the next section.
The only previous work on this question~\cite{JansenW07} has shown the following three results for the number $T$ of iterations taken to find the optimum: 
\begin{itemize}[leftmargin=10mm]
\item[(i)] If $\alpha \ge \epsilon n$ for any positive constant $\epsilon$, then \[E[T] = O(n \log n).\]
\item[(ii)]If $\alpha = o(n)$, then \[E[T] = \Omega(\alpha 2^{n/3\alpha}).\]
\item[(iii)]$E[T]$ is polynomial in $n$ if and only if $\alpha = \Omega(n / \log n)$.
\end{itemize}
%This first work clearly shows that a relatively large value of $\alpha$ is necessary to efficiently optimize \om. 

Our main result, see Theorem~\ref{thm:onemax} below, significantly extends this state of the art. Different from the previous work, it is tight apart from lower order terms for all $\alpha = \omega(\sqrt n)$ and thus, in particular, for the phase transition between $(1 + o(1)) n \ln n$ and runtimes exponential in $\frac n\alpha$. 

Moreover, it implies that the best possible 
runtime of $(1 \pm o(1)) n\ln n$ is obtained 
for $\alpha \ge \frac{n}{\ln \ln n}$, which 
characterizes the optimal parameter settings for the MA on \om. 

Our result also implies the known result that the runtime is polynomial in $n$ if and only if $\alpha = \Omega(\frac{n}{\log n})$, however, we also make precise the runtime behavior in this critical phase: For all $\alpha = \frac 1c \frac{n}{\ln n}$, the runtime is 
\[(1\pm o(1)) \frac 1c n^{c+1} (\ln n)^{-1}.\]

Our methods would also allow to prove results for smaller values of $\alpha$, but in the light of the previously shown $\exp(\Omega(n/\alpha))$ lower bound, these appear less interesting and consequently we do not explore this further.

\begin{theorem}\label{thm:onemax}
  Let $T$ be the runtime of the MA  with $\alpha=\omega(\sqrt{n})$ on \onemax. Then 
  \[E[T] = (1 \pm o(1)) n \ln(n) + \mathds{1}_{\alpha \le n} (1 \pm  o(1)) \alpha e^{n/\alpha}.\] 
\end{theorem}

The proof of our theorem uses state-of-the-art methods for the runtime analysis that were not available at the time of \citet{JansenW07}. In particular, we 
use so-called drift arguments, see \citet{Lengler20bookchapter}, to show that the MA quickly reaches the  
\emph{equilibrium point} of $\kappa = n-\ceil{\tfrac{n}{\alpha+1}}$ one-bits in the proof of the following Theorem~\ref{thm:posdrift}.

\begin{theorem}~\label{thm:posdrift}
Let $d(x)$ be the fitness distance (and Hamming distance) of $x$ to the optimum, in other words, the number of zeros in~$x$ and $x^{(t)}$ denote the current solution at the end of iteration $t$ (and $x^{(0)}$ the random initial solution).
Let $k^* = \frac{n}{\alpha+1}$ and $k = \lceil k^* \rceil$. Then the first time~$T$ such that the Metropolis algorithm with parameter $\alpha$ finds a solution~$x^{(T)}$ with $d(x^{(T)}) \le k$ satisfies 
\[
E[T] \le  \frac{\alpha}{\alpha+1} n (\ln(n)+1).
\] 
If $k = o(n)$, then we also have $E[T] \ge (1 - o(1)) n \ln(\frac nk)$. 
\end{theorem}
\begin{proof}
By definition of the algorithm, the probability for reducing the fitness distance from a solution $x$ with $i$ zero-bits is
\[
p_i^- := \frac{i}{n}
\]
and the probability for increasing the distance is
\[
p_i^+ := \frac{n-i}{\alpha n}.
\]

Let $D_t = d(x^{(t)})$ for convenience. We see that the expected progress in one iteration satisfies
\begin{equation}\label{eq:driftD}
E[D_t - D_{t+1} \mid D_t] = p_{D_t}^- + \frac 1 \alpha \, p_{D_t}^+  = D_t\left(\frac 1n + \frac 1{\alpha n}\right) - \frac 1 \alpha \, .
\end{equation} 

To derive a situation with multiplicative drift, we regard a shifted version of the process $(D_t)$. Let $X_t = D_t - k^*$ for all $t$. By~\eqref{eq:driftD}, we have
\begin{align*}
E[X_t - X_{t+1} \mid X_t] &= E[D_t - D_{t+1} \mid D_t] \\
&= D_t \left(\frac 1n + \frac 1{\alpha n}\right) - \frac 1 \alpha\\
&= \left(X_t + \frac {n}{\alpha+1}\right) \left(\frac 1n + \frac 1{\alpha n}\right) - \frac 1 \alpha \\
&= X_t \left(\frac 1n + \frac 1{\alpha n}\right) =: X_t \delta,
\end{align*}
that is, we have an expected multiplicative progress towards $0$ in the regime $X_t \ge 0$ (which is the regime $D_t \ge k^*$). 

To apply the multiplicative drift theorem, we require a process in the non-negative numbers, having zero as target, and such that the smallest positive value is bounded away from zero. For this reason, we define $(Y_t)$ by $Y_t = 0$ if $X_t < k + 1 - k^* =: \ymin$ and $Y_t = X_t$ otherwise (in other words, for $D_t \ge k+1$, the processes $(X_t)$ and $(Y_t)$ agree, and we have $Y_0=0$ otherwise). Since $(X_t)$ changes by at most one per step and since $\ymin \ge 1$, in an iteration $t$ such that $X_t = Y_t \ge \ymin$ we have $Y_{t+1} \le X_{t+1}$ and thus $E[Y_t - Y_{t+1} \mid Y_t] \ge E[X_t - X_{t+1} \mid X_t] = X_t \delta = Y_t \delta$, that is, we have the same multiplicative progress. We can thus apply the multiplicative drift theorem from~\cite{DoerrJW12algo} (also found as Theorem~11 in the survey~\cite{Lengler20bookchapter}) and derive that the first time $T$ such that $Y_T = 0$ satisfies $E[T] \le \frac{1 + \ln(n/\ymin)}{\delta} \le \frac{\alpha}{\alpha+1} n (\ln(n) + 1)$ as claimed.

For the lower bound, we argue as follows, very similar to the proof of~\cite[Proposition~5]{JansenW07}. Let $D_0$ be the fitness distance of the initial random search point. We condition momentarily on a fixed outcome of $D_0$ that is larger than $k$. Consider in parallel a run of the randomized local search heuristic RLS \cite{DoerrDY20} on \onemax, starting with a fitness distance of $D_0$. Note that this is equivalent to saying that we start a second run of the Metropolis algorithm with parameter $\alpha = \infty$. Denote the fitness distances of this run by $\tilde D_t$. This is again a Markov chain with one-step changes in $\{-1, 0, 1\}$, however, with transition probabilities $\tilde p_i^- = p_i^-$ and $\tilde p_i^+ = 0 \le p_i^+$. Consequently, a simple induction shows that $D_t$ stochastically dominates $\tilde D_t$. In particular, the first hitting time $\tilde T$ of $k$ of this chain is a lower bound for $T$, both in the stochastic domination sense and in expectation. We therefore analyze $\tilde T$. Since $\tilde D_t$ in each step either decreases by one or remains unchanged, we can simply sum up the waiting times for making a step towards the target, that is, $E[\tilde T] = \sum_{i = D_0}^{k-1} \frac 1 {p_i^-} = \sum_{i=D_0}^{k-1} \frac{n}{i} = n (H_{D_0} - H_{k-1})$, where $H_m := \sum_{i=1}^m \frac 1i$ denotes the $m$-th Harmonic number. Using the well-known estimate $\ln(m) \le H_m \le \ln(m)+1$, we obtain $E[\tilde T] \ge n(\ln(D_0) - \ln(k) - 1)$. Recall that this estimate was conditional on a fixed value of $D_0$. Since $D_0$ follows a binomial distribution with parameters $n$ and $\frac 12$, we have $D_0 \ge \frac n2 - n^{3/4}$ with probability $1 - o(1)$, and in this case, $E[T \mid D_0 \ge \frac n2 - n^{3/4}] \ge n(\ln(\frac n2 - n^{3/4} - 1) - \ln(k) - 1) = (1-o(1)) n \ln(\frac nk)$, where the last estimate exploits our assumption $k = o(n)$. Just from the contribution of this case, we obtain $E[T] \ge (1 - o(1)) E[T \mid D_0 \ge \frac n2 - n^{3/4}] = (1 - o(1)) n \ln(\frac nk)$.
\end{proof}

Once the number of one-bits has reached at least~$\kappa$, it is less likely to flip one of the few remaining zero-bits (which is necessary to make further progress) than to generate and accept an inferior solution. In this region of negative drift, we use a precise Markov chain analysis. 
This analysis considers the expected transition times $E_i$ from the level of $n-i$ to the level of $n-i+1$ one-bits and
makes heavy use of the recursion formula 
\[
E_i = \frac{n}{i} + \frac{n-i}{\alpha i}E_{i+1}
\]
that is a 
consequence of the Markov property and the 
transition probabilities $i/n$ and $\alpha (n-i)/n$ 
from the level of $n-i$ one-bits to a superior and 
inferior level, respectively.

We profit here from the fact that the optimization time when starting in an arbitrary solution in the negative drift regime is very close to the optimization time when starting in a solution that is a Hamming neighbor of the optimum. This runtime behavior, counter-intuitive at first sight, is caused by the fact that the apparent advantage of starting with a Hamming neighbor is diminished by that fact that (at least for $\alpha$ not too large) it is much easier to generate and accept an inferior solution than to flip the unique missing bit towards the optimum. We make this precise in the following theorem. Since it does not take additional effort, we formulate and prove this result for a range of starting points $k$ that extends also in the regime of positive drift. In this section, we shall use it only for $\ell = k = \lceil k^* \rceil$.

\begin{theorem} \label{thm:eke1}
For all $1 \le \ell \le \frac{2.5}{1+ 2.5/\alpha} \frac{n}{\alpha}$, we have
\[
E_1 \le E_\ell^0 \le (1 + O(\tfrac \alpha n)) E_1. 
\]
\end{theorem}
To estimate $E_1$, we use again elementary Markov arguments, but this time to derive an expression for $E_1$ in terms of $E_\ell$ for some $\ell$ sufficiently far in the regime with positive drift (Theorem~\ref{thm:bounds-E1}). Being in the positive drift regime, $E_\ell$ then can be easily bounded via drift arguments, which gives the final estimate for $E_1$ (Corollary~\ref{cor:e1}).

\begin{theorem} \label{thm:bounds-E1}
Let $\alpha \ge 1$ and $\ell = o(\sqrt n)$. Let 
\[
E_1^+ = n \left( \sum_{i=0}^{\ell-1} \left(\frac n \alpha\right)^i \frac{1}{(i+1)!} \right) + \left(\frac n \alpha \right)^\ell \frac 1 {\ell!} E_{\ell+1}.
\]
Then $(1-o(1)) E_1^+ \le E_1 \le E_1^+$.
\end{theorem}

By estimating $E_\ell$ for $\ell$ in the positive drift regime, we obtain the following estimate for $E_1$, which is tight apart from lower order terms when $\alpha = o(n)$. 

\begin{corollary}\label{cor:e1}
  Let $\alpha = \omega(\sqrt n)$. Then 
	\[
	(1- 2 \exp(-\tfrac 23 \tfrac n\alpha) - o(1)) \alpha e^{n/\alpha} \le 
	E_1 \le \alpha e^{n/\alpha}.
	\]
	If $\alpha \ge 2n$, then $E_1 \le 2n$.
\end{corollary}

\begin{proof}[Proof of Theorem~\ref{thm:onemax}]
  Let $k = \lceil \frac{n}{\alpha+1} \rceil$. Let $T_k$ be the first time that a solution $x$ with $d(x) \le k$ is found. By Theorem~\ref{thm:posdrift}, we have $E[T_k] \le (1+o(1))n \ln(n)$. Since $\alpha = \omega(1)$ and thus $k = o(n)$, Theorem~\ref{thm:posdrift} also gives the lower bound $E[T_k] \ge (1 - o(1)) n \ln(\frac nk)$. 

	When $\alpha \ge n-1$, that is, $k = 1$, then by Corollary~\ref{cor:e1} the remaining expected runtime is $E_1 = O(n)$. Together with our estimates on $T_k$, this shows the claim $E[T] = (1 \pm o(1)) n \ln (n)$ for this case.
	
	Hence let $\alpha < n - 1$ and thus $k \ge 2$. Since $\alpha = \omega(\sqrt n)$ and we aim at an asymptotic result, we can assume that $n$, and thus $\alpha$, are sufficiently large. Then $k \le 2 \frac{n}{\alpha+1} \le \frac{2.5}{1+ 2.5/\alpha} \frac{n}{\alpha}$, that is, $k$ satisfies the assumptions of Theorem~\ref{thm:eke1}. By this theorem, the expectation of the remaining runtime satisfies $E_k^0 = (1+O(\frac{\alpha}{n})) E_1$. By Corollary~\ref{cor:e1}, $E_1 \le \alpha e^{n/\alpha}$. This shows an upper bound of $E[T] \le (1+o(1))n \ln(n) + (1+O(\frac{\alpha}{n}))  \alpha e^{n/\alpha}$. For $\alpha \ge \frac{n}{\ln \ln n}$, this is the claimed upper bound $(1+o(1)) n \ln(n)$, for $\alpha < \frac{n}{\ln \ln n}$, this is the claimed upper bound $(1 \pm o(1)) n \ln(n) + (1 \pm o(1)) \alpha e^{n/\alpha}$. 
	
	If remains to show the lower bound for $\alpha < n-1$. If $\alpha \ge \frac{n}{\ln \ln n}$ and thus $k = O(\log \log n)$, the lower bound 
    \begin{align*}
	   E[T_k] &\ge (1 - o(1)) n \ln(\tfrac nk) \\
    &= (1 - o(1)) n \ln(n) 
	\end{align*} suffices. For $\alpha < \frac{n}{\ln \ln n}$, we estimate $E[T] \ge E[T_k] + E_k^0 \ge E[T_k] + E_1 = (1 - o(1)) n \ln(\frac nk) + (1 - 2 \exp(- \frac 23 \frac n\alpha) - o(1)) \alpha e^{n/\alpha} = (1 - o(1)) n \ln(\frac nk) + (1 - o(1)) \alpha e^{n/\alpha}$, again using Theorem~\ref{thm:eke1} and Corollary~\ref{cor:e1}. For $\alpha \ge \frac{n}{\ln(n)}$, we have $\ln(\frac nk) = (1 - o(1)) \ln(n)$ and thus $E[T] \ge  (1 - o(1)) n \ln(n) + (1 - o(1)) \alpha e^{n/\alpha}$. For $\alpha \le \frac{n}{\ln(n)}$, we have $n \ln(n) = o(\alpha e^{n/\alpha})$, hence our claimed lower bound is $E[T] \ge (1 - o(1)) \alpha e^{n/\alpha}$, which follows trivially from the estimate $E[T] \ge (1 - o(1)) n \ln(\frac nk) + (1 - o(1)) \alpha e^{n/\alpha}$ just shown. 
\end{proof}

As an additional insight from the drift analysis, we obtain that a \onemax-value of at least $n - \lceil\frac{n}{\alpha + 1} \rceil$ is always obtained very efficiently (in expected time $(1+o(1)) n \ln n$ at most). Hence as long as $\alpha = \omega(1)$, an almost optimal solution of fitness $(1 - o(1)) n$ is found in that time. 

%We 
%shall need arguments from this precise analysis of the MA on \onemax in the following section.

\subsection{Cliff}
\label{sec:cliff}

We now analyze the runtime of the Metropolis on the generalize class of \cliff functions proposed in this work. Compared to the previous analysis on classic \cliff functions in~\cite{LissovoiOW23}, we overcome two additional technical obstacles. The first, obvious one, is that our class of \cliff functions comes with two parameters~$m$ and~$d$ and the MA has the parameter~$\alpha$. In this three-dimensional 
parameter space, complex interactions appear which make the runtime analysis challenging and the final results complex. The second is that we aim at tighter bounds than those shown in~\cite{LissovoiOW23} and at bounds that are informative also in the super-polynomial regime. For this reason, we cannot assume that $\alpha = \Omega(n / \log n)$ as otherwise the hill-climbing part would be super-polynomial. 

Nevertheless, we can describe parameter ranges leading to efficient runtimes, determine whether the runtime depends 
polynomially or exponentially on the parameters, and achieve rather tight bounds.

Our main result, Theorem~\ref{thm:ma-cliff} below, distinguishes between two cases according to the location of the cliff (\ie, the points 
having $n-m$ one-bits) in relation 
to the equilibrium point $n-\lceil \frac{n}{\alpha+1}\rceil$ mentioned above \wrt\  \onemax. In Part~\ref{thm:ma-cliff:m>k} of the theorem, 
the equilibrium point comes after the cliff point (\ie, all points before the cliff have a positive drift), while Part~\ref{thm:ma-cliff:m<k} covers the opposite case. Intuitively,
the distinction takes into account that overcoming the cliff without falling back before the cliff is especially unlikely if the cliff 
resides in the negative drift region. In Part~\ref{thm:ma-cliff:m>k}, the runtime bounds depend crucially on the transition time 
$E_{m-1}$ from $n-m+1$ to $n-m+2$ one-bits. Intuitively, this term arises since all neighbors of state~$n-m+1$ are improving, making 
a transition back to the cliff point of $n-m$ one-bits likely. Once the MA has reached at least $n-m+2$ one-bits, at least one worsening would have to be a
accepted to return to the cliff, which in turn makes it more likely to reach the global optimum from those states.

%\todo{motivate and explain the following theorem}
\begin{theorem} \label{thm:ma-cliff}
 Let~$T$ be the runtime of the MA with~$\alpha=\omega(\sqrt{n})$ on $\cliff_{d,m}$, where $m=o(\sqrt{n})$ and $d\ge1$. Let $k^* \coloneqq \frac{n}{\alpha+1}$ and $\beta\coloneqq \frac{2.5}{1+2.5/\alpha}(n/\alpha)$. Let $E_{m-1}$ denote the first hitting time of a state with $n-m+2$ one-bits from a state 
 with $n-m+1$ one-bits. Then:
\begin{enumerate}
    \item \label{thm:ma-cliff:m>k} 
If $k^* < m+1$, then
    \begin{align*}
   \expect{T} &\le\begin{cases}
    \left((1+ O(\frac{\alpha}{n}))\frac {\left(\frac n \alpha\right)^{m-2}} {(m-2)!} +1 + o(\tfrac{\alpha}{n}) \right)E_{m-1} & \mbox{}\\
      \hfill {\text{if } m-2\le \beta,} & \\
    \left((1+ O(\frac{\alpha}{n}))\frac {\left(\frac n \alpha\right)^{m-2}} {(m-2)!} +\tfrac 53 +o(\tfrac{\alpha}{n}) \right)E_{m-1} & \mbox{}\\
     \hfill {\text{if } m-2> \beta},
   \end{cases} \\
    \text{and}& \\
    \expect{T} &\ge \left(1-o(1)\right)\left(\frac {\left(\frac n \alpha\right)^{m-2}} {(m-2)!} +1 \right)E_{m-1},
\end{align*}
where 
\begin{align*}
&  (1-o(1)) \frac{n^2\alpha^{d-1}}{m(m-1)}\left(\alpha+\frac{n}{m+1}\right)  \le E_{m-1} \\ 
 & \le (1+o(1)) \frac{n^2\alpha^{d-1}}{m(m-1)}\left(\alpha+\frac{n}{(m+1)\frac{\alpha+1}{\alpha}-n/\alpha}\right).
\end{align*}
    \item \label{thm:ma-cliff:m<k} 
If $m+1 \le k^*$,  then
    \begin{align*}
        &\left(\frac 1{\sqrt{2\pi}e^{\alpha/(12n)}}-o(1)\right) \frac{\alpha^{d+2}e^{n/\alpha}}{ \sqrt{n/\alpha}} \le \expect{T} \\
        & \hspace{3cm}\le  (1-o(1))\alpha^{d+2}e^{n/\alpha}.
    \end{align*}
\end{enumerate}
\end{theorem}

The proof of the theorem starts out with similar methods like on \onemax, \ie, drift analysis 
and estimations of expected transitions times between neighboring states of MA. A big amount of 
additional complexity has to be introduced to analyze the MA when its search point is around the cliff. Here 
the MA tends to repeatedly jump down and up the cliff, which has to be handled by 
careful analyses of random walks. Also, the second case 
($m+1\le k^*$) has to be analyzed with different arguments 
than the first case since the drift at the cliff point 
switches sign accordingly. Again, the recursive expression 
for the transition times $E_i$ mentioned above plays 
a central role in this analysis.

We will develop simpler, but weaker bounds for the runtime formulas of the theorem in Subsection~\ref{sec:comparison-ma-oea}
where we compare the results from the two main theorems 
in this section.

To compare the Metropolis algorithm with evolutionary algorithms, we now also estimate the optimization time of the \ooea on \cliff functions. An expected runtime~of $\Theta(n^m)$  has already been proven in~\citet{PaixaoHST17} for the classic case $d=m-3/2$ and mutation rate $p = \frac 1n$. 

In the following theorem, we prove an upper bound on the optimization time of the \ooea with general mutation rate~$p$ on $\cliff_{d,m}$. Since our main aim is showing that the \oea in many situations is faster than the MA, we prove no lower bounds. We note that for $k$ or $p$ not too large, one could show matching lower bounds with the methods developed in~\citet{DoerrLMN17,BamburyBD21}.

\begin{theorem} \label{thm:ea-cliff}
Consider the \ooea with general mutation rate $0 < p < \frac 12$ optimizing $\cliff_{d,m}$ with~arbitrary $m$ and $1 \le d < m-1$. Then the expected optimization time is at most 
\begin{align*}
E[T] & \le p^{-1} (1-p)^{-n+1} (1 + \ln n)  \\
& \quad + \binom{m}{\floor{d}+2}^{-1} p^{-\floor{d}-2} (1-p)^{-n+\floor{d}+2}.
\end{align*}
Any $p$ minimizing this bound satisfies $p \le \frac{\floor{d}+2}{n}$. If $m = O(n^{1/2} / \log n)$ and $p = \frac{\lambda}{n}$ for some $0 < \lambda \le \floor{d}+2$, then this bound is $E[T] \le (1+o(1)) \frac{e^\lambda}{\lambda^{\floor{d}+2}} \binom{m}{\floor{d}+2}^{-1}n^{\floor{d}+2}$. This latter bound is minimized for $\lambda = \floor{d}+2$, which yields 
\[
E[T] \le (1+o(1))  \binom{m}{\floor{d}+2}^{-1} \left(\frac{ne}{\floor{d}+2}\right)^{\floor{d}+2}.
\]
\end{theorem}

The proof of the theorem uses the well-known fitness-level 
technique \cite{Wegener01} to analyze the expected 
time until the 
\ooea reaches the cliff point. Moreover, the same 
type of arguments is used to analyze the expected 
time to 
reach the global optimum after having reached 
the level of at least $n-m+\lfloor d\rfloor+2$ one-bits, \ie, 
a level to the right of the cliff having strictly larger 
 fitness than the cliff point. 
These expected times 
are essentially 
covered by the first term in the general bound on $\expect{T}$ in the theorem. The second term essentially accounts for the expected 
time to reach the level  $n-m+\lfloor d\rfloor+2$ 
from the cliff. Additional arguments are used 
if $d$ is an integer, 
\ie, 
the levels of $n-m$ and $n-m+d+1$ one-bits share the same
fitness. Here MA can jump back to the cliff point from 
the level of $n-m+d+1$ one-bits (and in fact such a jump back tends to be more likely than a step increasing the number of one-bits), so a random walk analysis
is employed to estimate the time to reach the level 
of $n-m+d+2$ one-bits  for the first time.

%along with a random walk analysis 

\subsection{Comparison of MA and \oea}
\label{sec:comparison-ma-oea}

Still considering the \cliff function, we shall now compare the bounds we have obtained
for the expected runtime of the~MA in 
Theorem~\ref{thm:ma-cliff} 
with the 
bounds on the runtime of the \oea 
from Theorem~\ref{thm:ea-cliff}. To this end, we will first 
investigate optimal parameter choices for $\alpha$ depending on 
the \cliff parameters~$m$ and~$d$ 
and compare it with the 
bound for the  \oea, both for the standard mutation 
probability $1/n$ and the optimized one $(\floor{d}+2)/n$.
%This approach will give useful insights on the choice of $\alpha$ 
%and indicate that the \oea is 
%superior, but will not provide 
%not a complete picture. \merk{so far}
%Afterwards, we supplement this 
%study with a general comparison of 
%bounds for all interesting parameter
%settings, proving that MA cannot be 
%more efficient than the \oea.
%
%\paragraph{Optimizing $\alpha$ in MA.}

Our bounds on the runtime of MA 
and 
\oea  are 
rather precise, but arithmetically 
complicated and not necessarily tight. Since we want to analyze 
how much faster the MA can be compared to the \oea, we 
compute parameter settings for $\alpha$ that make the lower 
bounds for the MA as small as possible. These minimized 
lower bounds will be contrasted with the upper bounds for the \oea. Again, 
we have to distinguish between the two main cases for~$m$ in relation to~$n$ and $\alpha$
that appear in Theorem~\ref{thm:ma-cliff}. 
%We restrict ourselves to the case 
%that $d$ is not an integer.\todo{cover integral $d$, too}

\textbf{Case~$\boldmath m+1 > \ceil{\tfrac{n}{\alpha+1}}$:} In this case, corresponding to Part~\ref{thm:ma-cliff:m>k} 
of Theorem~\ref{thm:ma-cliff},  we 
have a lower bound on the runtime of the MA of 
\begin{align*}
%& (1-o(1)) \left(1+\frac{\left(\frac{n}{\alpha}\right)^{m-2}}{(m-2)!}\right)\frac{n^2 \alpha^{d-1}}{m(m-1)}\left(\alpha + \frac{n}{m-1}\right) \\ & \ge 
%
(1-o(1))
\left(1+\frac{\left(\frac{n}{\alpha}\right)^{m-2}}{(m-2)!}\right)\frac{n^2 \alpha^{d}}{m(m-1)}.
\end{align*}
%and
%upper bound on the runtime 
%of
%\[
%\left(1+O(\alpha/n) \frac{(n/\alpha)^{m-2}}{(m-2)!} + 5/3 + o(1)\right) \frac{n^2 %\alpha^{d-1}}{m(m-1)} 
%\left(\alpha + \frac{n\alpha}{(m+1)(\alpha+1)-n}\right)
%\]
%\merk{Die letzte Klammer kann beliebig groß werden, wenn $m+1$ 
%nahe genug bei $n/(\alpha+1)$ liegt}
%Since $\alpha+1> n/m$ in this case, the last %parenthesis is $\Omega(\alpha)$. 
By computing the derivative of $\alpha^d \big(1+\frac{\left(\frac{n}{\alpha}\right)^{m-2}}{(m-2)!}\big)$, 
%%We look at the following function which is proportional
%to the last lower bound: 
%\[f(\alpha)=\alpha^d \left(1+\frac{\left(\frac{n}{\alpha}\right)^{m-2}}{(m-2)!}\right).
%\]
we find that the expression is first decreasing and then increasing in~$\alpha$ 
if $m-2\ge d$. We assume this 
condition on~$d$ now without analyzing the border case $d\in (m-2,m-1)$. Then the bound (up to 
a factor $1\pm o(1)$) is minimized 
for 
$
\alpha^* = n \left(\frac{m-d-2}{d(m-2)!}\right)^{1/(m-2)}.
$
% \[
% \alpha^* = n \left(\frac{m-d-2}{d(m-2)!}\right)^{1/(m-2)}.
% \]
For convenience, we assume $m=\omega(1)$ hereinafter and 
obtain \[\alpha^*=(1\pm o(1)) \frac{en}{m-2}.\]
%We look at the first parenthesis and distinguish between two subcases. If $p\coloneqq \frac{\left(\frac{n}{\alpha}\right)^{m-2}}{(m-2)!}>1$, 
%then the whole expression is decreasing in $\alpha$ if $m-2\ge d$. We assume this 
%condition on~$d$ now without analyzing the border case $d\in (m-2,m-1)$. 
%If  $p\le 1$, then the first parenthesis is $O(1)$ and the whole expression 
%is increasing in~$\alpha$. Hence, up to lower-order terms, the %optimal $\alpha$ is attained if 
%$\frac{\left(\frac{n}{\alpha}\right)^{m-2}}{(m-2)!} = 1$. 
%Solving this equation, 
%we obtain $\alpha=(1\pm o(1)) en/m$. 
Plugging this in our lower bound, we have an 
expected runtime for the MA of at least
$
(1-o(1)) \frac{n^2}{m^2} \alpha^d = (1-o(1)) e^{d}\left(\frac{n}{m}\right)^{d+2}.
$
% \[
% (1-o(1)) \frac{n^2}{m^2} \alpha^d = (1-o(1)) e^{d}\left(\frac{n}{m}\right)^{d+2}.
% \]

By comparison, the bounds for the \oea with the two mutation probabilities  $1/n$ and  $(\floor{d}+2)/n$ are no larger than 
\[
(1+o(1)) e\left(\frac{n(\floor{d}+2)}{m}\right)^{\floor{d}+2}  \text{and } 
%\left(\frac{ned}{m(\ceil{d}+1)}\right)^{\ceil{d}+1} 
 (1+o(1)) \left(\frac{ne}{m}\right)^{\floor{d}+2},
\]
respectively, 
where we used $\binom{a}{b}\ge (a/b)^b$. 
%Since $d<m-1$ and $m$ is an integer, we have $\ceil{d}\le m-1$. 
Hence, if $d$ is not an integer, 
the bound for the optimized \oea is by a factor 
$\Theta((n/m)^{\ceil{d}-d})$ 
smaller than for the MA; in the exceptional case of integral~$d$, the
bound for the MA is by 
a factor no larger than $(1+o(1))e^2$ smaller. The bound 
for the standard \oea loses at most a factor of order $O(d^d)$.   Note also that $d$ must be a small constant for efficient (polynomial) optimization times anyway. 
Hence, the optimized MA 
is not much faster than the standard \oea, 
while typically the optimized \oea is even faster than 
the optimized MA.

Intuitively, the value $\alpha^*\approx \frac{en}{m}$ 
says that the equilibrium point $\frac{n}{\alpha+1}$ is around $\frac{m}{e}$, 
\ie, the 
cliff is clearly in the positive drift region. Hence, it seems plausible that the true 
minimal expected runtime of the MA is obtained for  $\alpha$ falling into the present case. However, since we do not have a sufficiently precise, global expression for the runtime, we still have to 
consider the case with the cliff in the negative
drift region.

\textbf{Case~$\boldmath m+1\le \ceil{\tfrac{n}{\alpha+1}}$:} In this case, corresponding to  
Part~\ref{thm:ma-cliff:m<k} of Theorem~\ref{thm:ma-cliff}, 
the bound on the expected runtime 
of the MA is
at least 
\begin{equation}
\left(\frac{1}{\sqrt{2\pi}e^{1/12}}-o(1)\right)
\frac{\alpha^{d+2}e^{n/\alpha}}{\sqrt{n/\alpha}},
\label{eq:lower-bound-ma-b}
\end{equation}
where we have used $\alpha\le n$. 
%\merk{(comment: expression minimized for $\alpha=2n/(2d+5)$)} 
%and asymptotically at most 
%\[
%\alpha^{d+2}e^{n/\alpha}.
%\]
For given $n$ and $d$, the latter expression is first decreasing 
and then increasing in~$\alpha$, with the minimum taken at $\alpha^*=\tfrac{n}{d+5/2}$. Depending 
on the integrality  of $\tfrac{n}{\alpha+1}$ 
appearing in the case condition, 
this choice of~$\alpha$ 
may be slightly too big and violate the general 
assumption $d<m-1$; however, then
$\alpha^*\approx \tfrac{n}{d+3}$ can be chosen  to minimize the bound while meeting the condition. This will not essentially change the following reasoning.
%
%\tfrac{n}{\alpha+1}
%
%We now verify that this choice of $\alpha$ does %not violate other 
%assumptions from Theorem~\ref{thm:ea-cliff}. 
%Since   we have 
% $m+1\le k^* = 
%\lceil\tfrac{n}{\alpha+1}\rceil$, plugging in the %optimized $\alpha$ and 
%assuming $\tfrac{n}{\alpha+1}$ as fractional number 
%gives the condition 
%\[
%m \le \lceil\tfrac{n}{n/(d+5/2)+1}\rceil-1 = d+5/2 %- \frac{(d+5/2)^2}{n+d+5/2}.
%\]
%This condition meets the 
% general 
%assumption $d<m-1$ from the definition of the %\cliff 
%function if 
%$d\in [m-3/2+\frac{(d+5/2)^2}{n+d+5/2},m-1)$.
%However, if $\tfrac{n}{\alpha+1}$ is an integer, %we may need to adjust $\alpha$ slightly, \eg, 
%to $\alpha\approx n/(d+3)$, to minimize the bound %while meeting the condition. This will not %essentially change the following reasoning.
%
Plugging in $\alpha=\frac{n}{d+5/2}$ in \eqref{eq:lower-bound-ma-b} 
gives a lower bound for the MA of 
 \[
 \left(\frac{1}{\sqrt{2\pi}e^{1/12}}-o(1)\right)
\left(\frac{ne}{d+5/2}\right)^{d+2}
\]
then. 
The upper bounds for the \oea 
in Theorem~\ref{thm:ea-cliff} for mutation probabilities~$1/n$ and $(\floor{d}+2)/n$ %and $d\notin \N$ 
are no larger than 
\[
(1+o(1)) en^{\ceil{d}+1} \text{ and } (1+o(1)) \left(\frac{ne}{\floor{d}+2}\right)^{\floor{d}+2},
\]
respectively, where we simply estimated $\binom{m}{\floor{d}+2}^{-1}\le 1$. 
If $d$ is not an integer, the bound for the optimized \oea  turns out
lower than for the MA  by a factor $\Theta((n/d)^{\ceil{d}-d})$; if 
it is an integer, the bound for MA is at most by a constant 
factor of \[\sqrt{2\pi}e^{1/12} \left(\frac{d+2.5}{d+2}\right)^{d+2}\] smaller. 
%and is never bigger, using $d=o(n)$ and large enough~$n$.
Moreover, the bound for the \oea with standard mutation rate is at most by a factor 
of order $O(d^{d})$ bigger. 
%Note also that $d$ must be a small constant for %efficient (polynomial) optimization times. 
%Hence, the optimized MA cannot outperform the %optimized \oea on \cliff and 
%is not much slower than the standard \oea. 
Altogether, we have arrived at the same conclusions 
as in the previous case.

\section{Experiments}
\label{sec:experiments}
To supplement our theoretical results, we have run the MA, the \oea and a few related algorithms on different instances of the \cliff problem.\footnote{Details about the implementation can be found in \href{https://github.com/DTUComputeTONIA/MAvsEAonCliff}{https://github.com/DTUComputeTONIA/MAvsEAonCliff}}
More precisely, besides the MA with $\alpha \in \{20,30,40\}$ (good values in a preliminary experiment with broader range of $\alpha$),
we used the \oea both with the standard mutation 
probability~$1/n$ and the higher mutation 
probability~$\ceil{d+1}/n$, 
the Fast-\oea using heavy-tailed mutation from \citet{DoerrLMN17} (with 
parameter $\beta=1.5$), 
and the 
SD-\oea, using stagnation detection,  from~\citet{RajabiW22} (with parameters $R = n^3$ and the threshold value $(n/r)^r (n/(n-r))^{n-r} \ln (enR)$ for strength $r$). We ran these algorithms on \cliff functions with problem size $n=150$ and problem parameters $d=3$ and growing~$m \in \{8,12, 16,\dots,32\}$.

The runtimes depicted in Figure~\ref{fig:comparison-150} show that for all three value of $\alpha$, 
the MA is clearly slower than the other algorithms (among which the \oea with high mutation rate and Fast~\oea performed best). 

\begin{figure} 
    \centering
	\includegraphics[width=0.99\linewidth]{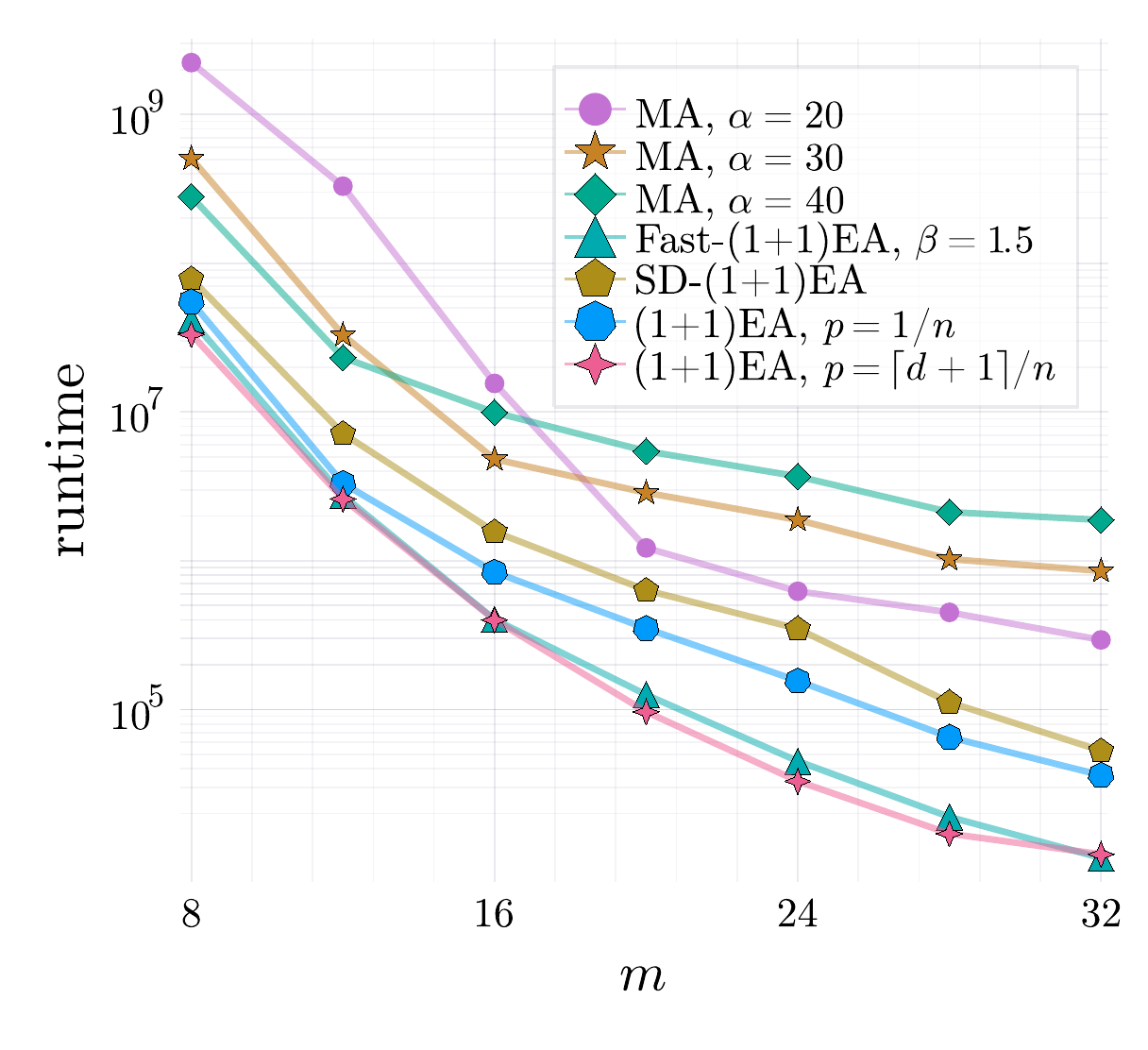}
	\caption{Comparison of MA with \oea and its variants on $\cliff_{m,d}$ for $n=150$, $d=3$ and increasing~$m$; averaged over 100 runs.}
	\label{fig:comparison-150}
\end{figure}

Since the EAs using global mutation apparently coped well with the valley of low fitness, we also tried the MA with these mutation operators instead of one-bit flips. In this set of experiments, we ran the standard MA, the MA using the  
standard bit mutation of the \ooea with mutation rate~$1/n$, and the MA using the heavy-tailed mutation 
of the Fast-\oea, all for $\alpha = 20$ (the most promising value for the smaller problem size $n=100$ which we used here). To 
investigate whether the ability of MA to accept 
worse search points was relevant, we included the \oea with mutation rate~$1/n$ in this comparison. In the runtime results presented in Figure~\ref{fig:combi-ma-ea-20}, the two global-mutation MAs overall perform better than the standard MA. The \oea might be the overall best choice, only occasionally mildly beaten by the heavy-tailed MA. 
% We note that increasing $\alpha$ 
% makes accepting worsenings less likely, so that the
% MA with global mutation resemble the corresponding variants of the \oea. If we increase $\alpha$, we obtain for all $m$ a ranking 
% consistent with the ranking in the interval $[16,28]$. 
% Altogether,
% we cannot point out clear cases where 
% the combination of global mutation and non-elitism 
% outperforms the elitist variants.
While these preliminary experiments do not suffice to draw final conclusions, they motivate as future work a closer analysis of MA with global mutation operators.

\begin{figure} 
    \centering
	\includegraphics[width=0.99\linewidth]{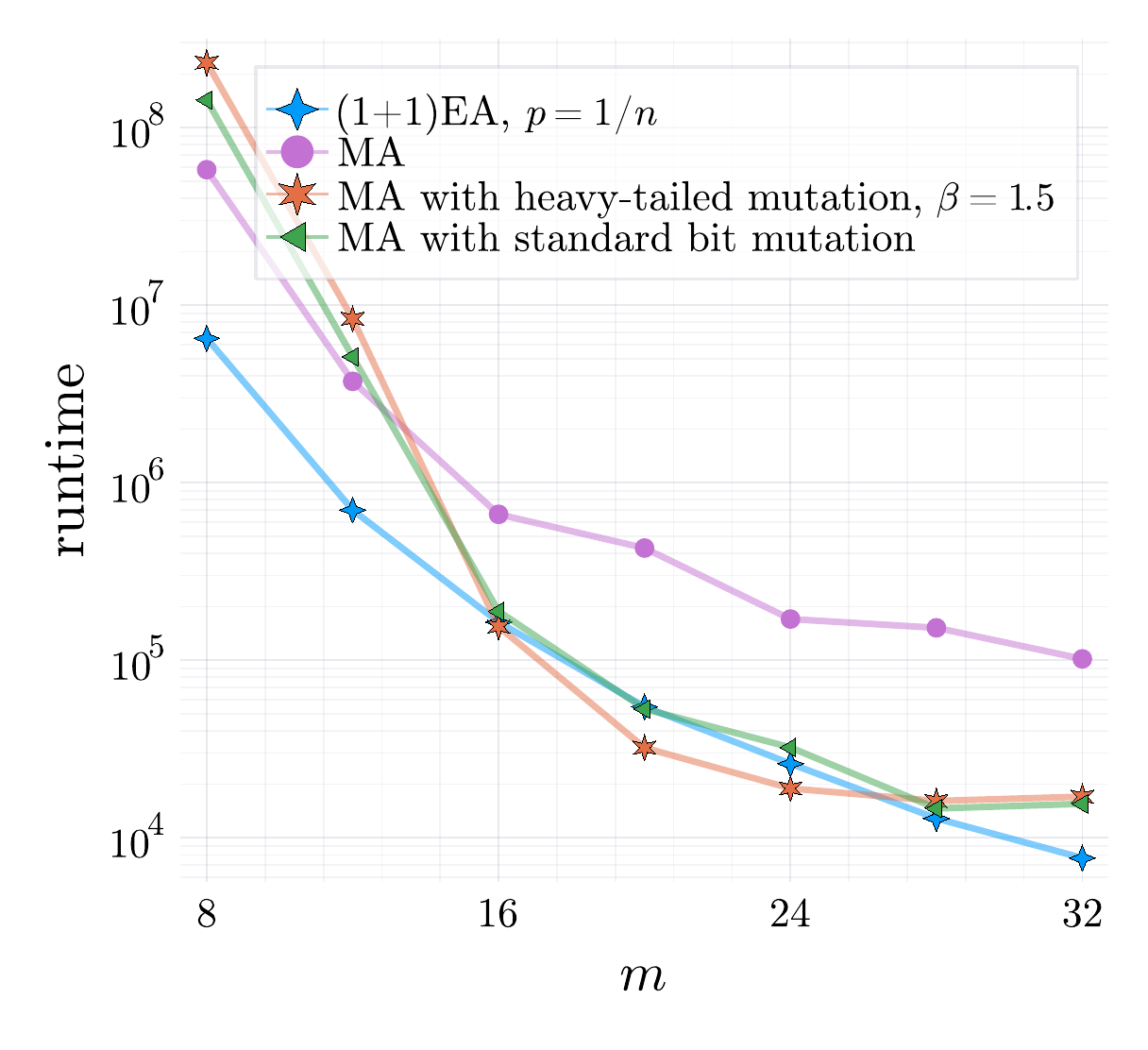}
	\caption{Comparison of \oea and different variants of MA, including  global mutation, on $\cliff_{m,d}$ for $n=100$, $\alpha=20$, $d=3$ and increasing~$m$; averaged over 50 runs.}
	\label{fig:combi-ma-ea-20}
	\end{figure}

\section{Conclusions}
\widowpenalty=10000
We have conducted a mathematical runtime analysis of the Metropolis algorithm (MA), a local search 
algorithm that with positive probability may accept worse solutions, on a generalization of the 
multimodal \cliff benchmark problem. This problem, which always has 
a gradient towards the global optimum except for one Hamming level, seems like a canonical candidate 
where MA can profit from its ability to accept inferior solutions. However, our mathematical runtime analysis 
has revealed that this intuition is not correct. The simple elitist \oea is
for many 
parameter settings   
faster than the MA, even for an optimal choice of the temperature parameter of the MA. This failed attempt to explain the effectiveness of the MA raises the question of what is the real reason for the success of the MA in many practical applications. 

The comparably good performance of algorithms using global mutation operators suggest to also use the MA with such operators. Our preliminary experimental results show that this can indeed be an interesting idea -- backing up these findings with a mathematical runtime analysis is an interesting problem for future research.

A second question of interest is to what extent simulated annealing, that is, the MA with a temperature decreasing over time, can improve the relatively weak performance of the classic MA on \cliff even with optimal choice of the temperature. From our understanding gained in this work, we are slightly pessimistic, mostly because again a relatively small temperature is necessary to reach the local optimum and then diving into the fitness valley is difficult, but definitely a rigorous analysis of this question is necessary to understand this question.

\begin{acks}
  Amirhossein Rajabi and Carsten Witt were supported by a grant from the Danish Council for Independent Research (DFF-FNU 8021-00260B). This work was also supported by a public grant as part of the
Investissements d'avenir project, reference ANR-11-LABX-0056-LMH,
LabEx LMH.
\end{acks}

\bibliography{alles_ea_master,ich_master}
\bibliographystyle{ACM-Reference-Format}

\ifthenelse{\boolean{arxiv}}
{
\onecolumn
\appendix

\section*{Appendix}

In the following two sections, we give detailed analyses including full mathematical proofs 
of the theorems stated in the Section \emph{Mathematical Runtime Analysis} of the main paper, dealing with the \onemax and 
\cliff functions, respectively.
%\section{Detailed Version of Mathematical Runtime Analysis Including Full Proofs}

\section{Analysis of OneMax -- Full Proofs}\label{sec:onemaxA}

In this section, we study the performance of the Metropolis algorithm on the \om benchmark. \om is the possibly best-studied benchmark in the theory of randomized search heuristics. For a given \emph{problem size} $n \in \N$, the \om function is the mapping $f : \{0,1\}^n \to \N$ defined by $f(x) = \|x\|_1 = \sum_{i=1}^n x_i$ for all $x = (x_1, \dots, x_n) \in \{0,1\}^n$. This is an easy benchmark representing problems or parts of problems where the gradient points into the direction of the global optimum $x^* = (1, \dots, 1)$. 

It is safe to say that $\Theta(n \log n)$ is a typical runtime of a randomized search heuristic optimizing $\om$. This runtime, more precisely, $(1 + o(1)) n \ln$ \cite{DoerrDY20}, was proven for the \emph{randomized local search} heuristic, a randomized hillclimber that flips a single random bit and accepts the new solution if it is at least as good as the previous one. Many simple evolutionary algorithms also solve \om in time $\Theta(n \log n)$ with suitable parameters, e.g., the mutation-based \mplea~\cite{Muhlenbein92,DrosteJW02,Witt06,AntipovD21algo}. 

It might appear surprising at first that it takes time $\Omega(n \log n)$ to find the correct value of $n$ bits for a simple function like \om, where the correct value each bit can be found from the discrete partial derivative at any search point (i.e., by comparing the fitness of the search point and the search point obtained by flipping this bit). The reason is that many randomized search heuristics flip bits chosen at random and then the so-called coupon-collector effect implies that it takes $\Omega(n \log n)$ time until each bit was flipped at least once. It is clear that this problem can be overcome, and an $O(n)$ runtime can be obtained, by flipping the bits in a given order, however, as shown in~\cite{DoerrFW10}, this can lead to unexpected difficulties when trying to design algorithms that not always flip single bits. Linear runtimes on \om have also been obtained via crossover-based EAs~\cite{DoerrD18,AntipovDK20}. The black-box complexity of \om, that is, the best performance a black-box optimization algorithm can have on the class of all functions isomorphic to \om, is $\Theta(n / \log n)$~\cite{ErdosR63}. Despite the fact that these faster performances have been shown for particular algorithms, it still appears appropriate to call $\Theta(n \log n)$ the typical runtime of a general-purpose search heuristic on \om. In fact, Lehre and Witt~\cite{LehreW12} have shown that any unary unbiased black-box algorithm, that is, any black-box algorithm that treats the bit positions $1, \dots, n$ and the bit values $0$ and $1$ in a symmetric fashion (unbiasedness) and that creates new solution only from one parent (unary), takes time at least $\Omega(n \log n)$ on \om. A precise tight bound of $(1 \pm o(1)) n \ln n$ was given in~\cite{DoerrDY20}.

We now conduct a precise analysis of how the Metropolis algorithm with different values of the parameter $\alpha$ performs on the \om problem. The only previous work on this question~\cite{JansenW07} has shown the following three results for the number $T$ of iterations taken to find the optimum.
\begin{itemize}
\item If $\alpha \ge \eps n$ for any positive constant $\eps$, then $E[T] = O(n \log n)$.
\item If $\alpha = o(n)$, then $E[T] = \Omega(\alpha 2^{n/3\alpha})$.
\item $E[T]$ is polynomial in $n$ if and only if $\alpha = \Omega(n / \log n)$.
\end{itemize}
This first work clearly shows that a relatively large value of $\alpha$ is necessary to efficiently optimize \om. 

Our main result (Theorem~\ref{thm:onemax}) is very precise analysis of the runtime of the Metropolis algorithm on \om showing that for all $\alpha = \omega(\sqrt n)$, we have 
\[
E[T] = (1\pm o(1)) n \ln(n) + (1\pm o(1)) \alpha \exp(\tfrac n\alpha).
\]
This result covers the most interesting regime describing the transition from polynomial to exponential runtimes. Our methods would also allow to prove results for smaller values of $\alpha$, but in the light of the previously shown $\exp(\Omega(n/\alpha))$ lower bound, these appear less interesting and consequently we do not explore this further.

Different from the previous work, our runtime result is tight apart from lower order terms for all $\alpha = \omega(\sqrt n)$ and thus, in particular, for the phase transition between $(1 + o(1)) n \ln n$ and runtimes exponential in $\frac n\alpha$. From this, we learn that we have a runtime of $(1 \pm o(1)) n \ln (n)$ if $\alpha \ge \frac{n}{\ln \ln n}$, but that the runtime becomes $\omega(n \log n)$ when $\alpha \le (1-\eps) \frac{n}{\ln \ln n}$ for any constant $\eps > 0$. Recall from above that $(1 \pm o(1)) n \ln n$ is the best runtime a unary unbiased black-box algorithm can have on $\onemax$ (and in fact any function $f : \{0,1\}^n \to \R$ with unique global optimum), so this insight characterizes the optimal parameter settings for the Metropolis algorithm on \om. 

Our result also implies the known result that the runtime is polynomial in $n$ if and only if $\alpha = \Omega(\frac{n}{\log n})$, however, we also make precise the runtime behavior in this critical phase: For all $\alpha = \frac 1c \frac{n}{\ln n}$, the runtime is $(1\pm o(1)) \frac 1c n^{c+1} (\ln n)^{-1}$.

We show this result not only because precise runtime results give a better picture of the performance of an algorithm, but also because our alternative analysis method gives additional insights on where this runtime stems from. In particular, we observe that a \om fitness of $\lceil n - \frac{n}{\alpha + 1} \rceil$ is always obtained very efficiently (in expected time $(1+o(1)) n \ln n$ at most). Hence as long as $\alpha = \omega(1)$, an almost optimal solution of fitness $(1 - o(1)) n$ is found in that time. A third motivation for this detailed analysis on \om is that we need similar arguments in the next section, where we study the performance of the Metropolis algorithm to see how well it copes with local optima.

\subsection{Preliminaries and Notation}\label{subsec:onemax:preliminaries}

From the symmetry of the Metropolis algorithm and the \om function, it is clear that all search points $x \in \{0,1\}^n$ having the same number of ones, that is, the same $\om$ value, and thus the same number of zeroes, that is, same distance 
\[
d(x) := n - \om(x)
\] from the optimum, behave equivalently. For this reason, let us, for all $i \in [0..n]$, denote by $L_i = \{x \in \{0,1\}^n \mid d(x) = i\}$ the set of search points in distance $i$. Since the Metropolis algorithm creates new solutions by flipping single bits, an iteration starting with a solution $x \in L_i$ for some $i$ can only end with a solution in $L_{i-1}$, $L_i$, or $L_{i+1}$. By definition of the algorithm, the probability for reducing the fitness distance from a solution $x \in L_i$ (that is, creating an offspring in $L_{i-1}$) is
\[
p_i^- := \frac{i}{n}
\]
and the probability for increasing the distance (that is, creating an offspring in $L_{i+1}$ and accepting it as new solution) is
\[
p_i^+ := \frac{n-i}{\alpha n}.
\]
We note that this notation is different from the one used in~\cite{JansenW07}, where the notation was based on the fitness and not on the distance. Hence our $p_i^+$ equals the $p_{n-i}^-$ used there. We prefer to work with the distance since the more critical part of the optimization process is close to the optimum.

With the transition probabilities just defined, we can use simple Markov chain arguments to, in principle, compute the expected runtime. For $i \ge j$, denote by $E_i^j$ the expected time the Metropolis algorithm (with some given parameter $\alpha$ suppressed in this notation) takes to find a solution in $L_j$ when started with a solution in $L_i$. We abbreviate $E_i = E_i^{i-1}$. Then, by elementary properties of Markov processes, 
\begin{equation}\label{eq:markov}
E_i^j = \sum_{\ell = i}^{j+1} E_\ell.
\end{equation}
From the one-step equation $E_i = 1 + p_i^+ (E_{i+1} + E_i) + (1 - p_i^+ - p_i^-) E_i$, we derive the following equation, which was also used in~\cite{JansenW07}.
\begin{equation}\label{eq:onestep}
E_i = \frac 1 {p_i^-} + \frac{p_i^+}{p_i^-} E_{i+1}.
\end{equation}

The analysis of the runtime of the Metropolis algorithm on \om in~\cite{JansenW07} was solely based on the above two equations (with, of course, non-trivial estimates of the arising sums). In this work, we partially take a different route by separately analyzing the part of the process in which the algorithm has a positive expected fitness gain per iteration. This is when the fitness distance is still large and thus is it easy to find improving solutions. In this part of the process, we can conveniently use multiplicative drift analysis, a tool presented a few years after~\cite{JansenW07}. For the remainder of the process, we use arguments similar to those in~\cite{JansenW07}, however, we profit from the fact that we need to cover only a smaller range of fitness levels.

\subsection{Pseudo-linear Time in the Regime with Positive Drift}\label{subsec:onemax:positive}

We start our analysis with the part of the process where the expected progress per iteration is positive. We recall that we denote by $d(x)$ the fitness distance (and Hamming distance) of $x$ to the optimum, in other words, the number of zeros in~$x$. Recalling that $x^{(t)}$ denotes the current solution at the end of iteration $t$ (and $x^{(0)}$ the random initial solution), and defining $D_t = d(x^{(t)})$ for convenience, we see that the expected progress in one iteration satisfies
\begin{equation}\label{eq:driftD}
E[D_t - D_{t+1} \mid D_t] = p_{D_t}^- + \frac 1 \alpha \, p_{D_t}^+ = \frac{D_t}{n} - \frac{n-D_t}{\alpha n} = D_t\left(\frac 1n + \frac 1{\alpha n}\right) - \frac 1 \alpha \, .
\end{equation} 
In particular, the expected progress is positive when $D_t > \frac{n}{\alpha+1} =: k^*$ and negative when $D_t < k^*$. An expected progress towards a target can be translated into estimates on the hitting time of this target, usually via so-called drift theorems~\cite{Lengler20bookchapter}, and this is our approach to show the following result.

\begin{theorem}~\label{thm:climb-pos-drift}\label{thm:posdrift}
Let $k^* = \frac{n}{\alpha+1}$ and $k = \lceil k^* \rceil$. Then the first time $T$ such that the Metropolis algorithm with parameter $\alpha$ finds a solution $x^{(T)}$ with $d(x^{(T)}) \le k$ satisfies 
\[
E[T] \le  \frac{\alpha}{\alpha+1} n (\ln(n)+1).
\] 
If $k = o(n)$, then we also have $E[T] \ge (1 - o(1)) n \ln(\frac nk)$. 
\end{theorem}

\begin{proof}
To derive a situation with multiplicative drift, we regard a shifted version of the process $(D_t)$. Let $X_t = D_t - k^*$ for all $t$. By~\eqref{eq:driftD}, we have
\begin{align*}
E[X_t - X_{t+1} \mid X_t] &= E[D_t - D_{t+1} \mid D_t] \\
&= D_t \left(\frac 1n + \frac 1{\alpha n}\right) - \frac 1 \alpha\\
&= \left(X_t + \frac {n}{\alpha+1}\right) \left(\frac 1n + \frac 1{\alpha n}\right) - \frac 1 \alpha \\
&= X_t \left(\frac 1n + \frac 1{\alpha n}\right) =: X_t \delta,
\end{align*}
that is, we have an expected multiplicative progress towards $0$ in the regime $X_t \ge 0$ (which is the regime $D_t \ge k^*$). 

To apply the multiplicative drift theorem, we require a process in the non-negative numbers, having zero as target, and such that the smallest positive value is bounded away from zero. For this reason, we define $(Y_t)$ by $Y_t = 0$ if $X_t < k + 1 - k^* =: \ymin$ and $Y_t = X_t$ otherwise (in other words, for $D_t \ge k+1$, the processes $(X_t)$ and $(Y_t)$ agree, and we have $Y_0=0$ otherwise). Since $(X_t)$ changes by at most one per step and since $\ymin \ge 1$, in an iteration $t$ such that $X_t = Y_t \ge \ymin$ we have $Y_{t+1} \le X_{t+1}$ and thus $E[Y_t - Y_{t+1} \mid Y_t] \ge E[X_t - X_{t+1} \mid X_t] = X_t \delta = Y_t \delta$, that is, we have the same multiplicative progress. We can thus apply the multiplicative drift theorem from~\cite{DoerrJW12algo} (also found as Theorem~11 in the survey~\cite{Lengler20bookchapter}) and derive that the first time $T$ such that $Y_T = 0$ satisfies $E[T] \le \frac{1 + \ln(n/\ymin)}{\delta} \le \frac{\alpha}{\alpha+1} n (\ln(n) + 1)$ as claimed.

For the lower bound, we argue as follows, very similar to the proof of~\cite[Proposition~5]{JansenW07}. Let $D_0$ be the fitness distance of the initial random search point. We condition momentarily on a fixed outcome of $D_0$ that is larger than $k$. Consider in parallel a run of the randomized local search heuristic RLS \cite{DoerrDY20} on \onemax, starting with a fitness distance of $D_0$. Note that this is equivalent to saying that we start a second run of the Metropolis algorithm with parameter $\alpha = \infty$. Denote the fitness distances of this run by $\tilde D_t$. This is again a Markov chain with one-step changes in $\{-1, 0, 1\}$, however, with transition probabilities $\tilde p_i^- = p_i^-$ and $\tilde p_i^+ = 0 \le p_i^+$. Consequently, a simple induction shows that $D_t$ stochastically dominates $\tilde D_t$. In particular, the first hitting time $\tilde T$ of $k$ of this chain is a lower bound for $T$, both in the stochastic domination sense and in expectation. We therefore analyze $\tilde T$. Since $\tilde D_t$ in each step either decreases by one or remains unchanged, we can simply sum up the waiting times for making a step towards the target, that is, $E[\tilde T] = \sum_{i = D_0}^{k-1} \frac 1 {p_i^-} = \sum_{i=D_0}^{k-1} \frac{n}{i} = n (H_{D_0} - H_{k-1})$, where $H_m := \sum_{i=1}^m \frac 1i$ denotes the $m$-th Harmonic number. Using the well-known estimate $\ln(m) \le H_m \le \ln(m)+1$, we obtain $E[\tilde T] \ge n(\ln(D_0) - \ln(k) - 1)$. Recall that this estimate was conditional on a fixed value of $D_0$. Since $D_0$ follows a binomial distribution with parameters $n$ and $\frac 12$, we have $D_0 \ge \frac n2 - n^{3/4}$ with probability $1 - o(1)$, and in this case, $E[T \mid D_0 \ge \frac n2 - n^{3/4}] \ge n(\ln(\frac n2 - n^{3/4} - 1) - \ln(k) - 1) = (1-o(1)) n \ln(\frac nk)$, where the last estimate exploits our assumption $k = o(n)$. Just from the contribution of this case, we obtain $E[T] \ge (1 - o(1)) E[T \mid D_0 \ge \frac n2 - n^{3/4}] = (1 - o(1)) n \ln(\frac nk)$.
\end{proof}

We note that there is a non-vanishing gap between our upper and lower bound in the theorem above when $k = n^{\Omega(1)}$. The reason, most likely, is the argument used in the lower bound proof that the Metropolis algorithm cannot be faster than randomized local search, which ignores any negative effect of accepting inferior solutions. For our purposes, the theorem above is sufficient, since for all but very large values of $\alpha$ (where the gap is negligible) the runtime of the Metropolis algorithm is dominated by the second part of the optimization process starting from a solution $x$ with $d(x) = k$. The reason why we could not prove a tighter bound for all values of $\alpha$ is that the existing multiplicative drift theorems for lower bounds, e.g., Theorem~2.2 in~\cite{Witt13} or Theorem 3.7 in~\cite{DoerrKLL20}, either are not applicable to our process or necessarily lead to a constant-factor gap to the upper bound obtained from multiplicative drift. Applying the variable drift theorem from~\cite{DoerrFW11} to the process $Z_t = \min\{Y_s \mid s \le t\}$ appears to be a promising way to overcome these difficulties, but since we do not need such a precise bound, we do not follow this route any further.

\subsection{Progress Starting the Equilibrium Point }

In the regime with negative drift, we use elementary Markov chain arguments to estimate runtimes. We profit here from the fact that the optimization time when starting in an arbitrary solution in the negative drift regime is very close to the optimization time when starting in a solution that is a Hamming neighbor of the optimum. This runtime behavior, counter-intuitive at first sight, is caused by the fact that the apparent advantage of starting with a Hamming neighbor is diminished by that fact that (at least for $\alpha$ not too large) it is much easier to generate and accept an inferior solution than to flip the unique missing bit towards the optimum. We make this precise in the following theorem. Since it does not take additional effort, we formulate and prove this result for a range of starting points $k$ that extends also in the regime of positive drift. In this section, we shall use it only for $\ell = k = \lceil k^* \rceil$.

\begin{theorem} \label{thm:eke1}
For all $1 \le \ell \le \frac{2.5}{1+ 2.5/\alpha} \frac{n}{\alpha}$, we have
\[
E_1 \le E_\ell^0 \le (1 + O(\tfrac \alpha n)) E_1. 
\]
\end{theorem}

\begin{proof}
  By equation~\eqref{eq:onestep} and the values for $p_i^+, p_i^-$ computed earlier, we see that 
\begin{equation}\label{eq:recursive-Ei}
    E_i = \frac{n}{i} + \frac{n-i}{\alpha i}E_{i+1},
\end{equation}
for all $i \in [1..n-1]$. By omitting the first summand, we obtain $E_{i+1} \le  \frac{\alpha i}{n-i} E_i$, and an elementary induction yields $E_{j+1} \le \alpha^j \frac{1 \cdot 2 \dots j}{(n-j) \dots (n-1)} E_1$ for all $j \in [1 .. n-1]$. Using the estimate $j! \le 3 \sqrt j (\frac je)^j$ stemming from a sharp version of Stirling's formula due to Robbins~\cite{Robbins55} (also stated as Theorem~1.4.10 in~\cite{Doerr20bookchapter}), we estimate, for $j \in [1..\ell]$,
\begin{align*}
\alpha^j \frac{1 \cdot 2 \dots j}{(n-j) \dots (n-1)} 
& \le 3 \frac{\alpha}{n-1} j^{3/2} \left(\frac{\alpha (j-1)}{e (n-j)}\right)^{j-1} \\
&  \le 3 \frac{\alpha}{n-1} j^{3/2} \left(\frac{\alpha \ell}{e (n-\ell)}\right)^{j-1} \\
&  \le 3 \frac{\alpha}{n-1} j^{3/2} \left(\tfrac{2.5}{e}\right)^{j-1},
\end{align*}
where we note that our assumption $\ell \le \frac{2.5}{1+ 2.5/\alpha} \frac{n}{\alpha}$ is equivalent to $\frac{\alpha \ell}{e (n-\ell)} \le \frac{2.5}{e}$. 

By~\eqref{eq:markov}, we have 
\begin{align*}
E_\ell^0 
& = E_1 + \sum_{j=1}^{\ell-1} E_{j+1} \\
& \le E_1 + \sum_{j=1}^{\ell-1} \alpha^{j} \frac{1 \cdot 2 \dots j}{(n-j )\dots (n-1)} E_1\\
& \le E_1 + \sum_{j=1}^{\ell-1} 3 \frac{\alpha}{n-1} j^{3/2} \left(\tfrac{2.5}{e}\right)^{j-1} E_1\\
& = (1 + O(\tfrac \alpha n)) E_1,
\end{align*}
where we exploited that the series $\sum_{j=1}^\infty j^A B^j$ converges for all constants $A \in \R$ and $b \in (0,1)$. Note that the first line in this set of equations also shows our elementary lower bound $E_\ell^0 \ge E_1$. 
\end{proof}

From Theorems~\ref{thm:posdrift} and~\ref{thm:eke1}, both applied with $k = \lceil k^* \rceil$, we know the runtime of the Metropolis algorithm on \onemax except that we do not yet understand $E_1$. This is what we do now.

\subsection{Estimating $E_{1}$}

To estimate $E_1$, we use again elementary Markov arguments, but this time to derive an expression for $E_1$ in terms of $E_\ell$ for some $\ell$ sufficiently far in the regime with positive drift (Theorem~\ref{thm:bounds-E1}). Being in the positive drift regime, $E_\ell$ then can be easily bounded via drift arguments, which gives the final estimate for $E_1$ (Corollary~\ref{cor:e1}).

\begin{theorem} \label{thm:bounds-E1}
Let $\alpha \ge 1$ and $\ell = o(\sqrt n)$. Let 
\[
E_1^+ = n \left( \sum_{i=0}^{\ell-1} \left(\frac n \alpha\right)^i \frac{1}{(i+1)!} \right) + \left(\frac n \alpha \right)^\ell \frac 1 {\ell!} E_{\ell+1}.
\]
Then $(1-o(1)) E_1^+ \le E_1 \le E_1^+$.
\end{theorem}

\begin{proof}
We first show that for all $\ell \in [0..n-1]$, we have 
\begin{equation}\label{eq:expansion}
E_1 = \sum_{i=0}^{\ell-1} \frac{n \cdot (n-1) \dots (n-i)}{\alpha^i (i+1)!} + \frac{(n-1) \dots (n-\ell)}{\alpha^\ell \ell!} E_{\ell+1}.
\end{equation}
This is trivially true for $\ell = 0$ (when using the convention that an empty sum evaluates to zero and an empty product evaluates to one). Assume now that equation~\eqref{eq:expansion} is true for some $\ell \in [0..n-2]$. Together with equation~\eqref{eq:recursive-Ei}, we compute
\begin{align*}
E_1 
& = \sum_{i=0}^{\ell-1} \frac{n \cdot (n-1) \dots (n-i)}{\alpha^i (i+1)!} + \frac{(n-1) \dots (n-\ell)}{\alpha^\ell \ell!} E_{\ell+1}\\
& = \sum_{i=0}^{\ell-1} \frac{n \cdot (n-1) \dots (n-i)}{\alpha^i (i+1)!} + \frac{(n-1) \dots (n-\ell)}{\alpha^\ell \ell!} \left(\frac{n}{\ell+1} + \frac{n - (\ell+1)}{\alpha(\ell+1)} E_{\ell+2}\right)\\
& = \sum_{i=0}^{\ell} \frac{n \cdot (n-1) \dots (n-i)}{\alpha^i (i+1)!} + \frac{(n-1) \dots (n-(\ell+1))}{\alpha^{\ell+1} (\ell+1)!} E_{\ell+2},
\end{align*}
which shows the equation also for $\ell+1$. By induction, the equation holds for all $\ell \in [0..n-1]$.

From equation~\eqref{eq:expansion}, we immediately obtain $E_1 \le E_1^+$. For $i = o(\sqrt n)$, we estimate $(n-1) \dots (n-i) = n^i \prod_{j=1}^i (1 - \frac in) \ge n^i (1 - \frac 1n \sum_{j=1}^i j) = n^i (1 - \frac{i(i-1)}{2n}) = n^i (1-o(1))$, where the inequality uses an elementary generalization of Bernoulli's inequality sometimes called Weierstrass inequality (see, e.g.,~\cite[Lemma~1.4.8]{Doerr20bookchapter}). This shows the lower bound $E_1 \ge (1 - o(1)) E_1^+$.
\end{proof}

By estimating $E_\ell$ for $\ell$ in the positive drift regime, we obtain the following estimate for $E_1$, which is tight apart from lower order terms when $\alpha = o(n)$. 

\begin{corollary}\label{cor:e1}
  Let $\alpha = \omega(\sqrt n)$. Then 
	\[
	(1- 2 \exp(-\tfrac 23 \tfrac n\alpha) - o(1)) \alpha e^{n/\alpha} \le 
	E_1 \le \alpha e^{n/\alpha}.
	\]
	If $\alpha \ge 2n$, then $E_1 \le 2n$.
\end{corollary}

\begin{proof}
  Let $\ell = \lceil 3 \frac{n}{\alpha} \rceil$. Since $\ell = o(\sqrt n)$, we can use Theorem~\ref{thm:bounds-E1} with this $\ell$ to estimate $E_1$. Let $X$ denote a random variable following a Poisson distribution with parameter $\lambda := \frac n \alpha$. Then
	\begin{align*}
	n \left( \sum_{i=0}^{\ell-1} \left(\frac n \alpha\right)^i \frac{1}{(i+1)!} \right) 
	&= \alpha \left( \sum_{i=0}^{\ell-1}  \frac{\lambda^{i+1}}{(i+1)!} \right) \\
	&= \alpha e^\lambda \left( \sum_{i=0}^{\ell}  \frac{\lambda^{i} e^{-\lambda}}{i!} - e^{-\lambda} \right)  \\
	&= \alpha e^\lambda \left( \Pr[X \le \ell] - e^{-\lambda} \right).
	\end{align*}
	
	To estimate the second summand in Theorem~\ref{thm:bounds-E1}, we first note that $(\frac n \alpha ) ^\ell \frac 1 {\ell!} \le (\frac{ne}{\ell\alpha})^\ell \le (\frac e 3)^{3\lambda}$ follows from the estimate $\ell! \ge (\frac \ell e )^\ell$. To bound $E_{\ell+1}$, we observe from equation~(\ref{eq:driftD}) that the drift of the fitness distance $D_t$, whenever $D_t \ge \ell+1$, satisfies $E[D_t - D_{t+1}] = D_t (\frac 1n + \frac 1{\alpha n}) - \frac 1 \alpha \ge \frac {\ell+1}n - \frac 1\alpha \ge \frac 2\alpha$, using $\ell \ge 3 \frac n\alpha$ in the last estimate. Consequently, we have an additive drift of at least $\frac 2\alpha$ for $D_t \in [\ell+1..n]$, and thus the additive drift theorem of He and Yao~\cite{HeY01} (also found as Theorem~2.3.1 in~\cite{Lengler20bookchapter}) yields that the expected time it takes to reach a $D_t$ value of $\ell$ when starting in $\ell+1$ is at most $E_{\ell+1} \le \frac \alpha 2$.

  Putting these three estimates together, we obtain an upper bound of 
	\begin{align*}
	E_1 & \le n \left( \sum_{i=0}^{\ell-1} \left(\frac n \alpha\right)^i \frac{1}{(i+1)!} \right) + \left(\frac n \alpha \right) ^\ell \frac 1 {\ell!} E_{\ell+1} \\
	& \le \alpha e^\lambda \left( \Pr[X \le \ell] - e^{-\lambda} \right) + \left(\frac e 3\right)^{3\lambda} \cdot \frac \alpha 2 \le \alpha e^\lambda.
	\end{align*}
	
	For the lower bound, we note that a Poisson random variable $Z$ with parameter $\lambda$ satisfies the Chernoff-type bound $\Pr[Z \ge \lambda + \gamma] \le \exp(-\frac{\gamma^2}{2(\lambda+\gamma)})$ for all $\gamma \ge 0$, see~\cite[Section~2.2]{BoucheronLM13}. Consequently, with $\gamma = 2\lambda$, we obtain $\Pr[X \le \ell] \ge 1 - \exp(-\frac 23 \lambda)$. This gives a lower bound of 
	\begin{align*}
	E_1 &\ge (1-o(1)) n \left( \sum_{i=0}^{\ell-1} \left(\frac n \alpha\right)^i \frac{1}{(i+1)!} \right) \\
	& \ge (1 - o(1)) \alpha e^\lambda \left( \Pr[X \le \ell] - e^{-\lambda} \right) \\
	& \ge (1- \exp(-\tfrac 23 \lambda) - \exp(-\lambda) - o(1)) \alpha e^{\lambda}.
	\end{align*}
	%
	%
	%Putting these three estimates together gives the claim for $\alpha = o(n)$. For $\alpha = \Omega(n)$, we have $k = \Theta(n)$ and thus $n ( \sum_{i=0}^{k-1} (\frac n \alpha)^i \frac{1}{(i+1)!}) = \Theta(n)$ and $(\frac n \alpha ) ^k \frac 1 {k!} = \Theta(1)$. Similar as above, we compute $E[D_t - D_{t+1}] = D_t (\frac 1n - \frac 1{\alpha n}) = \Omega(\frac 1n)$. Again, these three estimates give the claim in the case $\alpha = \Omega(n)$.
	For the case $\alpha \ge 2n$, we use again the additive drift argument. Whenever $D_t \ge 1$, we have $E[D_t - D_{t+1} \mid D_t] = D_t(\frac 1n + \frac 1 {\alpha n}) - \frac 1\alpha \ge \frac 1n - \frac 1\alpha \ge \frac 1 {2n}$. Hence the additive drift theorem bounds the expected time $E_1$ to reach $D_t = 0$ starting from $D_t = 1$ by $1 / \frac 1 {2n} = 2n$. 
	%\merk{put additive and multiplicative drift into the preliminaries}
\end{proof}

\subsection{A Tight Estimate for the Total Optimization Time }

From the partial results proven so far, we now obtain an estimate for the total runtime that is tight apart from lower order terms.

\begin{theorem}\label{thm:onemax}
  Let $T$ be the runtime of the Metropolis algorithm with parameter $\alpha$ on the \onemax function defined on bit strings of length~$n$. Let $\alpha = \omega(\sqrt n)$. Then 
  \[E[T] = (1 \pm o(1)) n \ln(n) + \bbone_{\alpha \le n} (1 \pm  o(1)) \alpha e^{n/\alpha}.\] 
\end{theorem}

\begin{proof}
  Let $k = \lceil \frac{n}{\alpha+1} \rceil$. Let $T_k$ be the first time that a solution $x$ with $d(x) \le k$ is found. By Theorem~\ref{thm:posdrift}, we have $E[T_k] \le (1+o(1))n \ln(n)$. Since $\alpha = \omega(1)$ and thus $k = o(n)$, Theorem~\ref{thm:posdrift} also gives the lower bound $E[T_k] \ge (1 - o(1)) n \ln(\frac nk)$. 

	When $\alpha \ge n-1$, that is, $k = 1$, then by Corollary~\ref{cor:e1} the remaining expected runtime is $E_1 = O(n)$. Together with our estimates on $T_k$, this shows the claim $E[T] = (1 \pm o(1)) n \ln (n)$ for this case.
	
	Hence let $\alpha < n - 1$ and thus $k \ge 2$. Since $\alpha = \omega(\sqrt n)$ and we aim at an asymptotic result, we can assume that $n$, and thus $\alpha$, are sufficiently large. Then $k \le 2 \frac{n}{\alpha+1} \le \frac{2.5}{1+ 2.5/\alpha} \frac{n}{\alpha}$, that is, $k$ satisfies the assumptions of Theorem~\ref{thm:eke1}. By this theorem, the expectation of the remaining runtime satisfies $E_k^0 = (1+O(\frac{\alpha}{n})) E_1$. By Corollary~\ref{cor:e1}, $E_1 \le \alpha e^{n/\alpha}$. This shows an upper bound of $E[T] \le (1+o(1))n \ln(n) + (1+O(\frac{\alpha}{n}))  \alpha e^{n/\alpha}$. For $\alpha \ge \frac{n}{\ln \ln n}$, this is the claimed upper bound $(1+o(1)) n \ln(n)$, for $\alpha < \frac{n}{\ln \ln n}$, this is the claimed upper bound $(1 \pm o(1)) n \ln(n) + (1 \pm o(1)) \alpha e^{n/\alpha}$. 
	
	If remains to show the lower bound for $\alpha < n-1$. If $\alpha \ge \frac{n}{\ln \ln n}$ and thus $k = O(\log \log n)$, the lower bound $E[T_k] \ge (1 - o(1)) n \ln(\frac nk) = (1 - o(1)) n \ln(n)$ suffices. For $\alpha < \frac{n}{\ln \ln n}$, we estimate $E[T] \ge E[T_k] + E_k^0 \ge E[T_k] + E_1 = (1 - o(1)) n \ln(\frac nk) + (1 - 2 \exp(- \frac 23 \frac n\alpha) - o(1)) \alpha e^{n/\alpha} = (1 - o(1)) n \ln(\frac nk) + (1 - o(1)) \alpha e^{n/\alpha}$, again using Theorem~\ref{thm:eke1} and Corollary~\ref{cor:e1}. For $\alpha \ge \frac{n}{\ln(n)}$, we have $\ln(\frac nk) = (1 - o(1)) \ln(n)$ and thus $E[T] \ge  (1 - o(1)) n \ln(n) + (1 - o(1)) \alpha e^{n/\alpha}$. For $\alpha \le \frac{n}{\ln(n)}$, we have $n \ln(n) = o(\alpha e^{n/\alpha})$, hence our claimed lower bound is $E[T] \ge (1 - o(1)) \alpha e^{n/\alpha}$, which follows trivially from the estimate $E[T] \ge (1 - o(1)) n \ln(\frac nk) + (1 - o(1)) \alpha e^{n/\alpha}$ just shown. 
\end{proof}

\section{Analysis of Cliff -- Full Proofs}

We now examine how the Metropolis algorithm  behaves when optimizing a function with a local optimum.
Two well-established and well-studied benchmark functions to model situations with local optima are \jump~\cite{DrosteJW02} and \cliff~\cite{JagerskupperS07}. In this research, we investigate the Metropolis algorithm on \cliff instead of \jump for two primary reasons. Firstly, on \jump functions, since the difference between the fitness of the local optimum and its neighbors is of order~$n$, it is unlikely that the algorithm accepts such a fitness decrease. Also, the deceptive valley in the function \jump does not allow the search to get far from the local optimum because of accepting all improvements. We refer the interested reader to~\cite[Theorem~13]{LissovoiOW19} for a lower bound  on the optimization time of Metropolis on \jump. We should mention that the authors in~\cite{OlivetoPHST18} also studied the Metropolis algorithm on the function so-called \textsc{Valley}, which is a multimodal problem containing both increasing and decreasing valleys.

In contrast to \jump functions, for \cliff functions the valley of low fitness is more shallow and the fitness inside the valley is not deceptive, that is, the gradient is pointing towards the optimum. These properties could let them appear like an easy optimization problem for the Metropolis algorithm, but as our precise analysis for the full spectrum of temperatures will show, this is not true.

\cliff functions were originally defined with only one parameter determining the distance of the local optimum from the global optimum \cite{Jagerskupper07}. However, since the Metropolis algorithm is sensitive to function values when it accepts worse solutions, we are interested in analyzing different depths for the valley in the fitness function. That is why we define an additional parameter to express the depth of the cliff.

Let again $n \in \N$ denote the problem size. As before, we shall usually suppress this parameter from our notation. 
Let $m\in\N_{\ge 1}$ and $d\in \R_{>0}$ such that $m<n$ and $d<m-1$. Then the  function $\cliff_{d,m}$ is increasing as the number of one-bits of the argument increases except for the points with~$n-m$ one-bits, where the fitness decreases sharply by~$d$ if we add one more one-bit to the search point. Formally,
\[
\cliff_{d,m}(x) \coloneqq \left\{
    \begin{array}{ll}
        \ones{x}  & \mbox{if } \ones{x} \leq n- m,\\
        \ones{x} - d -1 & \mbox{otherwise}.
    \end{array}
\right.
\]
See Figure~\ref{fig:cliff} for a graphical sketch. Note that the original cliff function can be obtained with fixed parameter $d=m-3/2$.

\begin{figure}[ht!]
    \centering
    \includegraphics[width=0.7\linewidth]{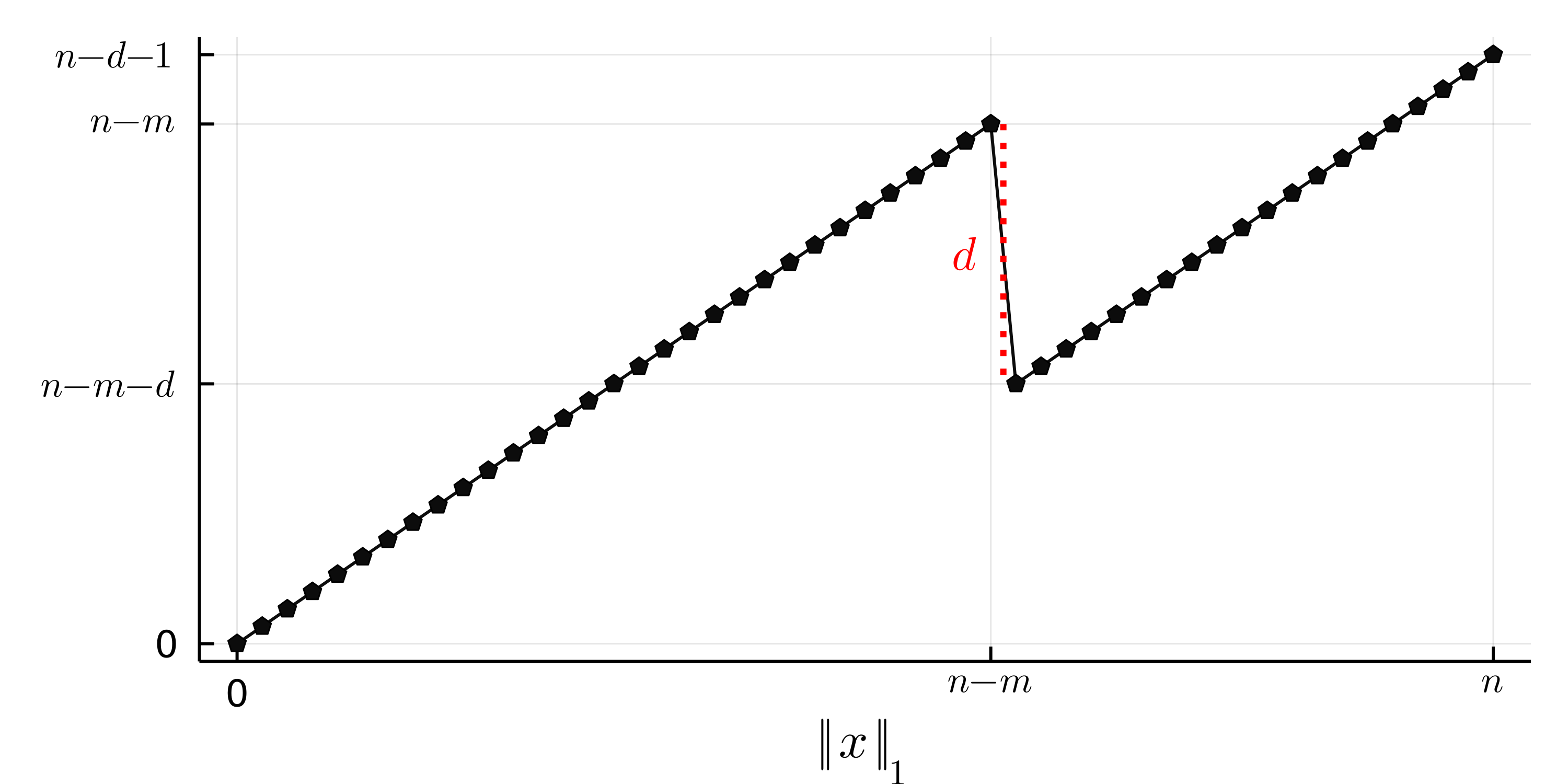}
    \caption{The function $\cliff_{d,m}$.}
    \label{fig:cliff}
\end{figure}

To the best of our knowledge, the only available analysis of the Metropolis Algorithm on \cliff functions is conducted in~\cite[Theorem~10]{LissovoiOW19}. On $\cliff_{d,m}$ with fixed~$d=m-3/2$, the authors show a lower bound of
\[\min\left\{\frac{1}{2}\cdot \frac{n-m+1}{m-1}\cdot \left(n/ \log n\right)^{m-3/2}, n^{\omega(1)}\right\}\] for some $\alpha\ge 1$. 
%\merk{what does this mean? OK, but we can't write this like this even if LOW write this.}

In this paper, we are interested in more rigorous bounds for the optimization time 
and a better understanding of the impact of two parameters~$d$ and~$m$, \ie, the depth and 
the local of the cliff on the runtime.

%the algorithm's behavior when confronted with local optima. 
%\merk{short description of our main result, not too technical}

% The parameter $\alpha$ has two opposing effects on the runtime: having a big $\alpha$ makes climbing the two hills easier, but will lose more time to overcome the local optimum.  

We recall some definitions used in Subsection~\ref{subsec:onemax:preliminaries}. Let $L_i$ be the set of search points with~$i$ zero-bits, \ie, $L_i\coloneqq \{x \in \{0,1\}^n \mid n-\ones{x} = i\}$. For $j<i$, let $E_i^j$ be the expected number of iterations the Metropolis algorithm spends to find a solution in $L_j$ when started with a solution in $L_i$. We write $E_i = E_i^{i-1}$. By the \emph{distance} $d(x)$ of a search point~$x$ we understand the Hamming distance to the global optimum~$1^n$, \ie, $d(x) = n-\ones{x}$.

Hereinafter, by writing \cliff, we mean the function~$\cliff_{d,m}$ defined in the beginning of the section for parameters $d,m$ clear from the context. There are two slopes in \cliff on which the algorithm has the same behavior as on \onemax. More precisely, for $i \notin \{m-1,m\}$, the expected time $E_{i}$ the Metropolis algorithm takes to find a solution with distance~$i-1$ when started with a solution with distance~$i$ follows Equation~\eqref{eq:recursive-Ei}, that is, we have
\begin{align}
    E_i = \frac{n}{i} + \frac{n-i}{\alpha i}E_{i+1}. \label{eq:cliff-E_i}
\end{align}

However, in the local optimum, \ie, solutions with $n-m$ one-bits, increasing the number of ones does not increase the fitness. In this case, $p^-_{m}$, denoting the probability of accepting a search point with distance~$m-1$ (\ie, a search point with~$n-m+1$ one-bits), equals $\alpha^{-d}m/n$, and $p^+_{m}$, denoting the probability of accepting a search point with distance~$m+1$, equals $\alpha^{-1}(n-m)/n$. Using Equation~\eqref{eq:onestep}, we obtain
\begin{align}
    E_m = \alpha^d \frac{n}{m} + \alpha^{d-1} \cdot \frac{n-m}{ m}E_{m+1}. \label{eq:cliff-E_m}
\end{align}

Finally, for the search points with distance~$m-1$, we have $p^-_{m-1}=\frac {m-1} n$ and $p^+_{m-1}= \frac{n-m+1}{n}$, resulting in
\begin{align}
E_{m-1} =  \frac{n}{m-1} +  \frac{n-m+1}{ m-1}E_{m}. \label{eq:cliff-E_m-1}
\end{align}

To ease our analysis of the optimization time~$T$, we shall assume that $m=o(n)$. Then a simple Chernoff bound argument shows that the initial search point is at a distance greater than~$m$ from the global optimum with high probability. Thus, we have
\[ \left(1-o(1)\right)\left(E_{m} + E_{m-1} + E_{m-2}^{0}\right) \le \expect{T} \le E_n^{m} + E_{m} + E_{m-1} + E_{m-2}^{0}. \]

Intuitively, the term~$E_{m-1}$ is one of the most influential terms on the total optimization time. The reason is that with the search points in this state, the algorithm goes back to the local optimum with $n-m$ one-bits $(1+o(1))n/m=\omega(1)$ times in expectation, where it has to again try many steps to accept a worse solution with fitness difference~$d$. Besides aforementioned event, the term~$E_1$, which is hidden in~$E^0_{m-2}$, also plays an important role in the optimization time as their corresponding search points have the least drift value or basically, the most negative drift value among all search points except the ones near the drop for some~$\alpha$.

We will even observe that~$E_1$ might impact on the total optimization time more significantly than~$E_{m-1}$ if all search points within the distance~$m$ have the negative drift, that is, the drop is at a distance less than the equilibrium~$k^*\coloneqq \frac{n}{\alpha+1}$.
In this case, if the algorithm is only one improvement away from  the global optimum, it might get back to the local optimum with high probability. That is why we are not really interested in~$E_{m-1}$, which is basically captured by~$E_1$ in this case. Contrariwise, when the equilibrium point~$k^*\coloneqq\frac{n}{\alpha+1}$ is on the second slope, it becomes essential to consider and analyze the role of~$E_{m-1}$ in the total optimization time, and this term cannot be ignored in the total optimization time without additional assumptions.

The following theorem is our main result in this section. This theorem analyzes the optimization time of Metropolis algorithm on $\cliff_{d,m}$ in two parts: where the search points with distance~$m+1$ are in the regime with positive drift (Part~\ref{thm:runtime-cliff:m>k}) or negative drift (Part~\ref{thm:runtime-cliff:m<k}). 

In Part~\ref{thm:runtime-cliff:m>k}, if the equilibrium point is far from the drop (\ie, $m-2 > \beta \approx 2.5k^* $), we have to use some additional arguments (Lemma~\ref{lem:climb-after-cliff}) as Theorem~\ref{thm:eke1} is only valid for $\ell \le \beta$. That is why there are two cases for the upper bound. We discuss this issue more comprehensively in the Subsection~\ref{sec:cliff:m>k}.

\begin{theorem} \label{thm:runtime-cliff}
 Let $k^* \coloneqq \frac{n}{\alpha+1}$ and $\beta\coloneqq \frac{2.5}{1+2.5/\alpha}(n/\alpha)$.
 Let~$T$ denote the first time Metropolis with~$\alpha=\omega(\sqrt{n})$ on $\cliff_{d,m}$ with $m=o(\sqrt{n})$ and $d\ge1$ finds the optimum point~$1^n$.
\begin{enumerate}
    \item \label{thm:runtime-cliff:m>k} If $k^* < m+1$, then
    \begin{align*}
   \expect{T} &\le\begin{cases}
    \left((1+ O(\frac{\alpha}{n}))\frac {\left(\frac n \alpha\right)^{m-2}} {(m-2)!} +1 + o(\alpha/n) \right)E_{m-1} & \text{if } m-2\le \beta, \\
    \left((1+ O(\frac{\alpha}{n}))\frac {\left(\frac n \alpha\right)^{m-2}} {(m-2)!} +5/3 +o(\alpha/n) \right)E_{m-1} & \text{if } m-2> \beta,
   \end{cases} \\
    \text{and}& \\
    \expect{T} &\ge \left(1-o(1)\right)\left(\frac {\left(\frac n \alpha\right)^{m-2}} {(m-2)!} +1 \right)E_{m-1},
\end{align*}
where \[
 (1-o(1)) \frac{n^2\alpha^{d-1}}{m(m-1)}\left(\alpha+\frac{n}{m+1}\right) \le E_{m-1} \le (1+o(1)) \frac{n^2\alpha^{d-1}}{m(m-1)}\left(\alpha+\frac{n}{(m+1)\frac{\alpha+1}{\alpha}-n/\alpha}\right).
\]
    
    \item \label{thm:runtime-cliff:m<k} If $m+1 \le k^*$,  then
    \[\left(\frac 1{\sqrt{2\pi}e^{\alpha/(12n)}}-o(1)\right) \frac{\alpha^{d+2}e^{n/\alpha}}{ \sqrt{n/\alpha}} \le \expect{T} \le  (1-o(1))\alpha^{d+2}e^{n/\alpha}.\]
\end{enumerate}
\end{theorem}

In Subsection~\ref{sec:cliff:m>k}, we discuss the optimization time for the part~\ref{thm:runtime-cliff:m>k} while the part~\ref{thm:runtime-cliff:m<k} is investigated in Subsection~\ref{sec:cliff:m<k}.
%In Subsection~\ref{sec:ea_vs_cliff}, we discuss and compare the Metropolis algorithm with the evolutionary algorithm \ooea.

\subsection{Progress When the Cliff is Below the Equilibrium Point} \label{sec:cliff:m>k}

In this subsection, we investigate the case that the drift at the search points with distance~$m+1$ is positive, that is, we prove part~\ref{thm:runtime-cliff:m>k} of Theorem~\ref{thm:runtime-cliff}.

In this case, the algorithm easily climbs up to the local optimum, that is, a search point with~$n-m$ one-bits, in $\Theta(n\log n)$ steps as shown in Subsection~\ref{subsec:onemax:positive}. Leaving the local optimum to a search point closer to the global optimum is difficult, that is, the time~$E_m$ is large, but since the algorithm from a search point in distance $m-1$ often moves back to the local optimum, we have this fact $E_m=o\left( E_{m-1}\right)$.

Regarding $E_{m-1}$, using recurrence relations obtained at the beginning of the section, we obtain a relation between $E_{m-1}$ and $E_{m+1}$. Also, since the drift at the search points with distance~$m+1$ is positive, we have a closed form for $E_{m+1}$ via the drift theorem, resulting in the following closed form for $E_{m-1}$.

\begin{lemma} \label{cliff-Em-1}
Let $k^*=\frac{n}{\alpha+1}$. For $k^* < m+1 = o(n)$, we have
\begin{align*}
 (1-o(1)) \frac{n^2\alpha^{d-1}}{m(m-1)}\left(\alpha+\frac{n}{m+1}\right) \le E_{m-1} \le (1+o(1)) \frac{n^2\alpha^{d-1}}{m(m-1)}\left(\alpha+\frac{n}{(m+1)\frac{\alpha+1}{\alpha}-n/\alpha}\right).
\end{align*}
\end{lemma}

\begin{proof}
Note that we have $\alpha>\frac{n}{m+1}-1=\omega(1)$ from our assumptions.

By Equation~\eqref{eq:cliff-E_m} and \eqref{eq:cliff-E_m-1}, we have
\begin{align*}
 E_{m-1} &= \frac{n}{m-1}+\frac{n-m+1}{m-1}E_m = \frac{n}{m-1}+\frac{n-m+1}{m-1}\left(\alpha^{d-1}\left(\frac {\alpha n}m+\frac{n-m}{m}E_{m+1}\right)\right).
\end{align*}
For the lower bound, since $m=o(n)$, we have
\begin{align}
    E_{m-1} & \ge \frac{n-m+1}{m-1}\left(\alpha^{d-1}\left(\frac {\alpha n}m+\frac{n-m}{m}E_{m+1}\right)\right) \nonumber \\
    & = (1-o(1)) \frac{n}{m-1}\left(\alpha^{d-1}\left(\frac {\alpha n}m+(1-o(1))\frac{n}{m}E_{m+1}\right)\right) \nonumber \\
    & = (1-o(1)) \frac{n^2}{m(m-1)}\left(\alpha^{d-1}\left(\alpha+E_{m+1}\right)\right), \label{eq:cliff-Em-1:lb}
\end{align}
and for the upper bound, 
\begin{align}
 E_{m-1} &\le \frac{n}{m-1}+\frac{n}{m-1}\left(\alpha^{d-1}\left(\frac {\alpha n}m+\frac{n}{m}E_{m+1}\right)\right) \nonumber \\
 & = \frac{n}{m-1}+\frac{n^2}{m(m-1)}\left(\alpha^{d-1}\left(\alpha+E_{m+1}\right)\right) \nonumber \\
 & = (1+o(1))\frac{n^2}{m(m-1)}\left(\alpha^{d-1}\left(\alpha+E_{m+1}\right)\right).\label{eq:cliff-Em-1:ub}
\end{align}

It remains to estimate~$E_{m+1}$. Using Equation~\eqref{eq:driftD}, the drift at distance~$m+1$ is positive and equals~$\Delta \coloneqq \frac{(m+1)(\alpha+1)-n}{\alpha n}$. Since the drift at larger distances is at least $\Delta$, by the additive drift theorem of He and Yao~\cite{HeY01} (also found as Theorem~2.3.1 in~\cite{Lengler20bookchapter}), the expected time to reach the distance~$m$ starting from the distance~$m+1$ is at most~$1/\Delta=\frac{n}{(m+1)\frac{\alpha+1}{\alpha}-n/\alpha}$. For the lower bound on~$E_{m+1}$, a necessary condition to reach the distance~$m$ from the distance~$m+1$ is that one of~$m+1$ zero-bits flips, happening with probability $(m+1)/n$. This upper bound on reaching the distance~$m$ holds in every step. Using the geometric distribution, we need at least~$n/(m+1)$ steps in expectation. Altogether, we have 
\begin{align} \label{eq:Em+1-closed}
    \frac{n}{m+1} \le E_{m+1} \le \frac{n}{(m+1)\frac{\alpha+1}{\alpha}-n/\alpha}.
\end{align}
%\merk{indeed, here we should argue better}

Replacing~$E_{m+1}$ in Equation~\eqref{eq:cliff-Em-1:lb} and \eqref{eq:cliff-Em-1:ub} with Equation~\eqref{eq:Em+1-closed}, we get the following bounds.
\begin{align*}
 (1-o(1)) &\frac{n^2\alpha^{d-1}}{m(m-1)}\left(\alpha+\frac{n}{m+1}\right) \le E_{m-1} \\
 &\le (1+o(1)) \frac{n^2\alpha^{d-1}}{m(m-1)}\left(\alpha+\frac{n}{(m+1)\frac{\alpha+1}{\alpha}-n/\alpha}\right).
\end{align*}
% \begin{align*}
%  E_{m-1} = (1\pm o(1)) \left( \frac{n^2}{m(m-1)}\alpha^{d-1}\left(\alpha+\frac{n}{(m+1)-n/\alpha}\right)\right).
% \end{align*}
\end{proof}

We note that the above estimate for $E_{m+1}$ would be exactly the same for the optimization process on \onemax since the time to go from distance $m+1$ to $m$ is not affected by the cliff. The reason why we could not use results from Section~\ref{sec:onemax} here is that there we did not analyze the $E_i$ separately, but used a simpler argument to analyze the sum of the first $E_i$. 

Now, we discuss how we estimate~$E_{m-2}^{0}$, \ie, the expected time to reach the global optimum from a search point located in the second position after the drop in the valley. Since in the valley, we have the same recurrence relation between~$E_i$ and~$E_{i+1}$ as for \onemax, we can use similar arguments as in Theorem~\ref{thm:eke1}. If we have~$m-2\le \beta \coloneqq \frac {2.5}{1+2.5/\alpha} \frac{n}{\alpha}$, the term~$E_{m-2}^{0}$ asymptotically equals~$E_1$. Otherwise, if~$m-2> \beta$, we only have the estimation $E_{\beta}^{0}=(1+o(\alpha/n))E_1$, so we also need to analyze~$E_{m-2}^{\beta}$. 

In the following lemma, we prove that the expected time to reach the distance $\beta$ starting from $m-2$, \ie, $E_{m-2}^{\beta}$, is at most by a constant factor larger than~$E_{m-1}$.

\begin{lemma}\label{lem:climb-after-cliff} Let $\beta \coloneqq  \frac{2.5}{1+2.5/\alpha}(n/\alpha)$.
For $\beta < m-2=o(\sqrt{n})$, we have
\[E_{m-2}^{\beta} \leq (2/3+o(1)) E_{m-1}.\]
\end{lemma}

For the proof of the previous lemma, we need a classical inequality from probability theory called 
Wald's inequality.

\begin{lemma}[Wald's inequality from~\cite{DoerrK15}]\label{lem:wald}
 Let $T$ be a random variable with a finite expectation, and let $X_1$, $X_2$, \dots be non negative random variables with $\expect{X_i\mid T\ge i}\le C$. Then
 \[\expect{\sum_{i=1}^{T}X_i }\le \expect{T}\cdot C.\]
 \end{lemma}

\begin{proofof}{Lemma~\ref{lem:climb-after-cliff}}
For $\beta\leq i \leq m-1 $, let $S_i$ be the event of that the algorithm starts from a search point with distance~$i$ and reaches a search point with distance~$\beta$ before it reaches distance~$m-1$. Reusing the notation~$x^{(t)}$ from Subsection~\ref{subsec:onemax:positive}, we let $X_t = n-\ones{x^{(t)}}$. If we define $U_a^b=\min\{t\mid X_t=b \text{ and } X_0=a\}$, $S_i$ is defined as the event that $U_i^{\beta} < U_i^{m-1}$.

% \[a_i = \Pr \left( \min \left\{t \mid X_t = \beta,X_0) = i \right\} > \min \left\{t \mid X_t = m-1, X_0 = i  \right\} \right), \]

In the first part of the proof, we aim at bounding~$\prob{S_{m-2}}$ from below. According to the definition, we have $\prob{S_{m-1}} = 0$, $\prob{S_{\beta}} = 1$, and for $ \beta+1\leq i \leq m-2$, using the law of total probability,
    \begin{align*}
        \prob{S_i} = p_i^-  \prob{S_{i-1}} + p_i^+  \prob{S_{i+1}} + (1- p_i^- - p_i^+)\prob{S_{i}},
    \end{align*}
    which can be rewritten as 
    \[\prob{S_{i}}-\prob{S_{i-1}} = \frac{p_i^+}{p_i^-} (\prob{S_{i+1}} - \prob{S_{i}}).\]
    By denoting $w_i \coloneqq \prod_{k=i}^{m-2} \frac{p_k^+}{p_k^-}$ and carrying out a simple induction, for all $\beta+1 \leq i \leq m-1$, we have
    \[\prob{S_{i}}-\prob{S_{i-1}} = w_i(\prob{S_{m-1}} - \prob{S_{m-2}}).\]
    Through a telescoping sum of the equations, we get
    \[\prob{S_{m-1}} - \prob{S_{\beta}} = \left(\prob{S_{m-1}} - \prob{S_{m-2}}\right)\sum_{i=\beta+1}^{m-1} w_i.\]
    Hence, using the fact that $\prob{S_{m-1}} = 0$, $\prob{S_{\beta}} = 1$, we have
    \[\prob{S_{m-2}} = \frac 1 {\sum_{i=\beta+1}^{m-1} w_i}. \]
    Furthermore, for $\beta+1\leq i \leq m-2$, since
    $\frac{p_i^+}{p_i^-} = \frac{n-i}{\alpha i} \leq \frac n { \alpha \cdot \beta}$, we obtain $w_i\le \left( \frac{n}{\alpha \beta} \right)^{m-i-1}$. Since $n/(\alpha\beta)< 1$, using the geometric series sum formula, we get
    \[ \sum_{i=\beta+1}^{m-1} w_i \leq \sum_{i=0}^{\infty} \left( \frac n {\alpha \beta} \right)^k  = \left(1-\frac{n}{\alpha \beta}\right)^{-1}, \]
    resulting in $\prob{S_{m-2}} \ge \left(1-\frac{n}{\alpha \beta}\right)=\frac{3+5/\alpha}{5}\ge 3/5$ through the geometric distribution.
    
    Now, in the second part of the proof, we estimate the time~$\tau$ for the process starting from a search point with distance~$m-2$ to reach a search point with distance either~$m-1$ or~$\beta$, that is,
    \[\tau \coloneqq \min \left\{ U_{m-2}^{m-1}, U_{m-2}^{\beta} \right\}.\] 
    To compute an upper bound on $E[\tau]$, we introduce a stopped process, $y^{(t)}$ defined as follows. We let $y^{(0)} = x^{(0)}$, and for $t \geq 1$,
    \[y^{(t)}  = \left\{
    \begin{array}{ll}
        x^{(t)}  & \mbox{if } X_t \neq m-1, \\
        y^{(t-1)} & \mbox{otherwise},
    \end{array}
\right.\]
and let $Y_t \coloneqq n-\ones{y^{(t)}},$ and $\tau' \coloneqq U_{m-2}^{\beta}$.
The process $y^{(t)}$ follows the same movements as $x^{(t)}$, except if $x^{(t)}$ gets to a point with distance~$m-1$, where it is stopped. With our definition, we see immediately that $\forall t \geq 0, Y_t \leq m-2$, and that $\tau \leq \tau'$.

To compute $E[\tau'],$ we compute the drift associated to $Y_t$, and we notice that this process has the same drift as the process associated to the optimization of \onemax by the MA ($y^{(t)}$ is not affected by the cliff because it can never reach it). We then use Theorem~\ref{thm:climb-pos-drift} to deduce that 
\[ E[\tau] \le E[\tau'] \leq \frac{\alpha}{\alpha+1} n(\ln n+1),\]
resulting in $E[\tau]=o(E_{m-1})$ through Lemma~\ref{cliff-Em-1} with $m=o(\sqrt{n})$.

Now, in the final part of the proof, we estimate~$E_{m-2}^{\beta}$ by using~$\prob{S_{m-2}}$ and~$\expect{\tau}$ which were bounded in the previous paragraphs. Let $I_t$ be the random variable denoting the number of iterations starting from a search point with distance~$m-2$ to reach a search point with distance~$m-1$, and thereafter again to reach a search point with distance~$m-2$. Then we have
    \begin{align} \label{eq:walds-climb-after-cliff}
        E_{m-2}^{\beta} \leq \expect{\sum_{t=1}^\ell I_t} + \expect{\tau\mid S_{m-2}}, 
    \end{align}    
    where $\ell$ is the number of times the algorithm reaches a search point with distance~$m-1$ before $\beta$ starting with a search point with distance~$m-2$. Since the assumptions of the Wald's inequality in Lemma~\ref{lem:wald} are satisfied, the first term in the right-hand side of the inequality equals~$\expect{\ell}C$, where based on the definition, we have $\expect{I_i \mid i \le \ell}=\expect{E_{m-1}}+\expect{\tau \mid \overline{S_{m-2}}}\eqqcolon C$. Also, using the geometric distribution, we have $\expect{\ell}+1=\prob{S_{m-2}}^{-1}$. Altogether, the right-hand side of Inequality~\eqref{eq:walds-climb-after-cliff} is bounded from above by
    \begin{align*}
    &\left( \prob{S_{m-2}}^{-1}-1 \right) \left(E_{m-1}+\expect{\tau\mid \overline{S_{m-2}}}\right) + \expect{\tau\mid S_{m-2}} \\ 
    &\quad=\left( \prob{S_{m-2}}^{-1}-1 \right)E_{m-1} +  \left( \frac{1-\prob{S_{m-2}}}{\prob{S_{m-2}}} \right)\expect{\tau\mid \overline{S_{m-2}}} + \expect{\tau\mid S_{m-2}} \\ 
    &\quad= \left( \prob{S_{m-2}}^{-1}-1 \right)E_{m-1}+\left( \frac{\prob{\overline{S_{m-2}}}\expect{\tau\mid \overline{S_{m-2}}} + \prob{S_{m-2}}\expect{\tau\mid S_{m-2}}}{\prob{S_{m-2}}} \right) \\ 
    &\quad= \left( \prob{S_{m-2}}^{-1}-1 \right)E_{m-1}+\left( \frac{\expect{\tau}}{\prob{S_{m-2}}} \right).
    \end{align*}
Using the bounds obtained on~$\prob{S_{m-2}}$ and~$\expect{\tau}$, we can finally conclude 
\[E_{m-2}^{\beta} \leq \left( 2/3 + o(1) \right) E_{m-1}.\]
\end{proofof}

Finally, by working out the informal arguments for the overall proof idea given at the beginning of this subsection, we rigorously prove the optimization time in Part~\ref{thm:runtime-cliff:m>k} of Theorem~\ref{thm:runtime-cliff} as follows.

\begin{proofof}{Theorem~\ref{thm:runtime-cliff} part~\ref{thm:runtime-cliff:m>k}} Let $z_0\coloneqq \ones{x^{(0)}}$ be the number of one-bits in the initial random search point. Then, using Chernoff's bound, for $m=o(n)$, $\prob{z_0 \le n-m}$, \ie, the probability that the initial search point is at a distance at least~$m$ from the optimum, is exponentially close to 1, more precisely $1-2^{-\Omega(n)}$. Therefore, we have
\begin{align}
\left(1-2^{-\Omega(n)}\right)\left(E_{z_0}^m+E_{m} + E_{m-1} + E_{m-2}^{0}\right) \le \expect{T} \le E_{z_0}^{m} + E_{m} + E_{m-1} + E_{m-2}^{0}. \label{eq:E[T]:part_a}
\end{align}
In the following paragraphs, we estimate each of the terms in the last expression and then finally bound $\expect{T}$.

\begin{itemize}

\item[$E_{z_0}^m$:] Since $m+1>\frac{n}{\alpha+1}$ and $\alpha=\omega(\sqrt{n})$, the drift as defined is positive, and by Theorem~\ref{thm:climb-pos-drift}, 
\[ (1-o(1))n\ln n\le E_{z_0}^m \le \frac{\alpha}{\alpha+1} n(\ln (n)+1). \]
Using Lemma~\ref{cliff-Em-1} and the fact that~$d\ge 1$, we have $E_{m-1}= \Omega(n^3/m^3)$, resulting in $E_{z_0}^m=o(E_{m-1})$ for $m=o(\sqrt{n})$.

\item[$E_m, E_{m-1}$:] By Equation \eqref{eq:cliff-E_m-1}, we have $E_m=o(E_{m-1})$, and using Lemma~\ref{cliff-Em-1}, we have $E_{m-1}$ as defined in the statement of the lemma.
    
\item[$E^0_{m-2}$:] We consider the following two cases. If $m-2 \le \beta$, we use Theorem~\ref{thm:eke1} since for all $i\in[1..m-2]$, the equation $E_i=\frac ni + \frac {n-i}{\alpha i}E_{i+1}$ holds due to Equation~\eqref{eq:cliff-E_i}, so all arguments in the proof of Theorem~\ref{thm:eke1} are still valid. This results in $E_{m-2}^{0} = (1+O(\alpha/n))E_1$. Otherwise, if $m-2>\beta$, we have $E^0_{m-2}=E^{\beta}_{m-2}+E^0_{\beta}$. By Lemma~\ref{lem:climb-after-cliff} we get $E_{m-2}^{\beta} \leq (2/3+o(1)) E_{m-1}$ and by Theorem~\ref{thm:eke1}, we have $E_{\beta}^{0} = (1+O(\alpha/n))E_1$ for the same reason as in the previous case. Therefore, we have $E_{m-2}^0 \ge E_1$ and
\begin{align*}
   E_{m-2}^0 \le \begin{cases}
            (1+O(\alpha/n))E_1 & \text{if } m-2\le \beta, \\
            (2/3+o(1))E_{m-1} + (1+O(\alpha/n))E_1 & \text{if } m-2> \beta.
   \end{cases}
\end{align*}
%  In order to compute $E_1$, as for $1\le i \le m-2$, the equation $E_i=\frac{n}{i}+\frac{n-i}{\alpha i}E_{i+1}$ holds and $m=o(\sqrt{n})$, we can use Theorem~\ref{thm:bounds-E1} for~$\ell = m-2$ and get $(1-o(1))E_1^+\le E_1\le E_1^+$, we have
 
 To compute $E_1$, since 
 $m=o(\sqrt{n})$ and the equation $E_i=\frac{n}{i}+\frac{n-i}{\alpha i}E_{i+1}$ holds for $1\le i \le m-2$, we can use Theorem~\ref{thm:bounds-E1}. By this theorem for $\ell = m-2$, we get $(1-o(1))E_1^+\le E_1\le E_1^+$ such that
\begin{align}
    E_1^+ & = n\left(\sum_{i=0}^{m-3}\left(\frac n \alpha\right)^i \frac 1 {(i+1)!}\right) + \left(\frac n \alpha\right)^{m-2}\frac 1 {(m-2)!} E_{m-1} \nonumber \\
    & = \alpha \left(\sum_{i=0}^{m-3}\left(\frac n \alpha\right)^{i+1} \frac 1 {(i+1)!}\right) + \left(\frac n \alpha\right)^{m-2}\frac 1 {(m-2)!} E_{m-1}. \label{eq:E1+}
\end{align}

We now compute the first summand in two cases according to~$\alpha$: $\alpha<n$ and $\alpha \ge n$. In the first case, which is $n/\alpha>1$, let us denote $f(k)\coloneqq\left(\frac n \alpha\right)^k \frac 1 {k!}$.
Since $f(k)/f(k-1)=\frac{n/\alpha}{k}$, $f(k)$ is increasing for $k< n/\alpha$. Thus, since we have $m+1 > \frac{n}{\alpha+1}$, the summation in the last expression (Equation~\eqref{eq:E1+}) can be bounded based on the largest term, which is the term of~$i=m-3$. Therefore, we get
\begin{align*}
    E_1^+ & \le \alpha (m-3) \left(\frac n \alpha\right)^{m-2} \frac 1 {(m-2)!} + \left(\frac n \alpha\right)^{m-2}\frac 1 {(m-2)!} E_{m-1} \\
    & \le \left(\frac n \alpha\right)^{m-2}\frac 1 {(m-2)!} \left( \alpha m + E_{m-1} \right)\\
    & =(1+o(1))\left(\frac n \alpha\right)^{m-2}\frac 1 {(m-2)!}  E_{m-1},
\end{align*}
where we have used $m=o\left(\sqrt{n}\right)$ and $E_{m-1}=\Omega(\alpha n^2/m^2)$ using Lemma~\ref{cliff-Em-1} with $d\ge1$.

For $n/\alpha \le 1$, the first summand in Equation~\eqref{eq:E1+} is $O(\alpha)$, so it is again asymptotically dominated by $E_{m-1}=\Omega(\alpha n^2/m^2)$ using Lemma~\ref{cliff-Em-1} with $d\ge1$. Thus, for $n/\alpha \le 1$, we have
\begin{align*}
    E_1^+ & = o(E_{m-1}) + \left(\frac n \alpha\right)^{m-2}\frac 1 {(m-2)!} E_{m-1}.
\end{align*}

For both cases $n/\alpha>1$ and  $n/\alpha\le 1$, we can conclude the upper bound
\begin{align*}
    E_1^+ & \le o(E_{m-1}) + (1+o(1))\left(\frac n \alpha\right)^{m-2}\frac 1 {(m-2)!} E_{m-1}.
\end{align*}

\end{itemize}

Altogether, by Equation~\eqref{eq:E[T]:part_a}, we can conclude $\expect{T} \ge (1-o(1)) E_{m-1}$, and
\begin{align*}
   \expect{T} \le \begin{cases}
            E_{z_0}^{m}+(1+o(1))E_{m-1}+(1+O(\alpha/n))E_1^+ & \text{if } m-2\le \beta, \\
            E_{z_0}^{m}+(1+o(1))E_{m-1}+ (2/3+o(1))E_{m-1}+ (1+O(\alpha/n))E_1^+ & \text{if } m-2> \beta,
   \end{cases}
\end{align*}
which gives us
\begin{align*}
   \expect{T} \le \begin{cases}
     \left((1+ O(\frac{\alpha}{n}))\frac {\left(\frac n \alpha\right)^{m-2}} {(m-2)!} +1 +o(\alpha/n) \right)E_{m-1} & \text{if } m-2\le \beta, \\
  \left((1+ O(\frac{\alpha}{n}))\frac {\left(\frac n \alpha\right)^{m-2}} {(m-2)!} +5/3 +o(\alpha/n) \right)E_{m-1} & \text{if } m-2> \beta,
   \end{cases}
\end{align*}
and for the lower bound, we get
\begin{align*}
    \expect{T} \ge \left(1-o(1) \right) \left(\left(\frac n \alpha\right)^{m-2}\frac 1 {(m-2)!} +1 \right)E_{m-1}.
\end{align*}

\end{proofof}

\subsection{Progress When the Cliff is in Negative Drift Region} \label{sec:cliff:m<k}

In this subsection, we analyze the optimization time of Metropolis on \cliff, when the drift at the search points with distance~$m+1$ is not positive. In other words, the algorithm reaches local optima no sooner than a search point with negative drift.

In this case, we shall argue (in the proof of Theorem~\ref{thm:runtime-cliff}, Part~\ref{thm:runtime-cliff:m<k}) that \[ \left(1-2^{-n}\right)E_1 < \expect{T} \le (1+O(\alpha/n))E_1,\]
that is, the optimization time is well described by the time taken to find the optimum from one of its Hamming neighbors. Here the lower bound stems from the fact that the algorithm has to visit a Hamming neighbor before finding the optimum except when the random initial solution is already the optimum.

For the upper bound, by elementary properties of Markov processes, we have
\[\expect{T} \le   E_n^{\ceil{\frac n{\alpha+1}}}+ E_{\ceil{\frac n{\alpha+1}}}^{m} + E_m + E_{m-1}^{0}.\]
Using Equation~\eqref{eq:cliff-E_i}, \eqref{eq:cliff-E_m}, and \eqref{eq:cliff-E_m-1}, we will show that $E_{i}$ is asymptotically larger than $E_{i-1}$ for the last three terms representing the search points in the negative drift region.

Therefore, the following lemma bounding $E_1$ plays an important role to estimate the optimization time. The $\alpha^{d+2}e^{n/\alpha}$ factor appearing in both bounds comes from the fact that the gap is reached and has to be overcome a repeated number of times due to the negative drift.
\begin{lemma}\label{lem:E1:m<kstar}
If $m \le \frac n{\alpha+1}-1$ and $\alpha=\omega(\sqrt{n})$, we have
 \[\left(\frac 1{\sqrt{2\pi}e^{\alpha/(12n)}}-o(1)\right) \frac{\alpha^{d+2}e^{\floor{n/\alpha}}}{ \sqrt{\floor{n/\alpha}}} \le E_1 \le \alpha^{d+2}e^{n/\alpha} + o(n).\]
\end{lemma}
\begin{proof}
Regarding the lower bound, using Equation~\eqref{eq:cliff-E_i}, we have $E_i = \frac ni + \frac {n-i}{\alpha i} E_{i+1} \ge \frac ni + \frac {n}{c\alpha i} E_{i+1}$ for $i\in [1..m-2]\cup [m+1..\floor{\frac n{\alpha}}]$, where $c=1+\floor{\frac{n}{\alpha}}/(n-\floor{\frac{n}{\alpha}})$ because
\begin{align*}
E_i &= \frac ni + \frac {n-i}{\alpha i} E_{i+1} = \frac ni + \frac {n}{(1+i/(n-i))\alpha i} E_{i+1} \\  &\ge \frac ni + \frac {n}{(1+\floor{\frac{n}{\alpha}}/(n-\floor{\frac{n}{\alpha}}))\alpha i} E_{i+1}  = \frac ni + \frac {n}{c\alpha i} E_{i+1}.
\end{align*}
Then, using the recursive formulas for $i\in [1..m-2]$, we achieve
\begin{align*}
E_1 &\ge \sum_{i=0}^{m-3} \frac n{(i+1)!} \left( \frac n {c\alpha} \right)^{i} +  \frac{1}{(m-2)!}\left( \frac n {c\alpha} \right)^{m-2}E_{m-1}.
\end{align*}
In the drop region, i.e., $i=m$ and $i=m-1$, using Equation~\eqref{eq:cliff-E_m} and~\eqref{eq:cliff-E_m-1}, we have
\begin{align*}
    &E_{m-1} \ge \frac n{m-1} + \frac {n}{c(m-1)} E_{m}, \\ 
    &E_{m} \ge \alpha^d \frac n{m}  + \alpha^d \frac {n}{c\alpha m} E_{m+1},
\end{align*}
which results in 
\begin{align*}
&E_1 \ge \sum_{i=0}^{m-2} \frac n{(i+1)!} \left( \frac n {c\alpha} \right)^{i} +  \frac{\alpha}{(m-1)!}\left( \frac n {c\alpha} \right)^{m-1}E_{m} ,
\end{align*}
and furthermore
\begin{align*}
&E_1  \ge \sum_{i=0}^{m-2} \frac n{(i+1)!} \left( \frac n {c\alpha} \right)^{i} +  \frac{n\cdot \alpha^{d+1}}{m!}\left( \frac n {c\alpha} \right)^{m-1}+  \frac{\alpha^{d+1}}{m!}\left( \frac n {c\alpha} \right)^{m}E_{m+1}.
\end{align*}
By using Equation~\eqref{eq:cliff-E_i} for $m+1\le i\le\floor{\frac{n}{\alpha}}$, we achieve

\begin{align*}
E_1 &\ge \sum_{i=0}^{m-2} \frac n{(i+1)!} \left( \frac n {c\alpha} \right)^{i} + \alpha^{d+1} \sum_{i=m-1}^{\floor{\frac{n}{\alpha}}-1}\frac{n}{(i+1)!}\left( \frac n {c\alpha} \right)^{i}
+  \frac{\alpha^{d+1}}{\floor{\frac{n}{\alpha}}!}\left( \frac n {c\alpha} \right)^{\floor{\frac{n}{\alpha}}}E_{\floor{\frac{n}{\alpha}}+1}.
\end{align*}
The last expression is bounded from below by
\begin{align*}
\alpha^{d+1} \sum_{i=m-1}^{\floor{\frac n\alpha}-1}\frac{n}{(i+1)!}\left( \frac n {c\alpha} \right)^{i} \ge c \alpha^{d+2} \sum_{i=m-1}^{\floor{\frac n\alpha}-1}\frac{1}{(i+1)!}\left( \frac n {c\alpha} \right)^{i+1} \ge c \alpha^{d+2} \frac{1}{\floor{n/\alpha}!}\left( \frac n {c\alpha} \right)^{\floor{n/\alpha}}.
\end{align*}

Since $\frac{n}{\alpha+1}-1>m$ and $m>0$, we have $n/\alpha>2$. Thus, using Stirling's formula (Theorem~1.4.10 in \cite{Doerr20bookchapter}), the last term is bounded from below by
\begin{align*}
    c \alpha^{d+2}\frac{1}{\sqrt{2\pi \floor{n/\alpha}}e^{\alpha/(12n)}}  \left( \frac{en}{c\alpha \floor{n/\alpha}} \right)^{\floor{n/\alpha}} \ge  \frac{\alpha^{d+2}}{c^{n/\alpha-1}\cdot \sqrt{2\pi \floor{n/\alpha}}e^{\alpha/(12n)}}  e^{\floor{n/\alpha}}.
\end{align*}
Since $n/\alpha=o(\sqrt{n})$, we have  \[c^{n/\alpha-1}=\left(1+\frac{\floor{\frac n{\alpha+1}}}{(n-\floor{\frac n{\alpha+1}})}\right)^{n/\alpha-1}\le \left(1+\frac{n/\alpha}{n-n/\alpha}\right)^{n/\alpha}\le e^{\frac{n}{\alpha(\alpha-1)}}=1+o(1),\]
where $\alpha=\omega(\sqrt{n})$. Then, we have
\begin{align*}
    E_1 \ge \left(\frac 1{\sqrt{2\pi}e^{\alpha/(12n)}}-o(1)\right) \frac{\alpha^{d+2}e^{\floor{n/\alpha}}}{ \sqrt{\floor{n/\alpha}}}.
\end{align*}
% Since $ n/\alpha\ge 1$, we can bound it from below by
% \begin{align*}
%     (0.367-o(1)) \frac{\alpha^{d+2}e^{n/\alpha}}{ \sqrt{n/\alpha}}.
% \end{align*}

Regarding the upper bound, using Equation~\eqref{eq:cliff-E_i}, for $i \in [1..m-2]\cup[m+1..\ceil{3n/\alpha}]$, we have $E_i = \frac ni + \frac {n-i}{\alpha i} E_{i+1} \le \frac ni + \frac {n}{\alpha i} E_{i+1}$.

Then, using the recursive formulas for $i\in [1..m-2]$, we achieve
\begin{align*} \label{eq:1tom-1}
E_1 &\le \sum_{i=0}^{m-3} \frac n{(i+1)!} \left( \frac n \alpha \right)^{i} +  \frac{1}{(m-2)!}\left( \frac n \alpha \right)^{m-2}E_{m-1}.
\end{align*}

In the drop region, i.e., $i=m$ and $i=m-1$, using Equation~\eqref{eq:cliff-E_m} and \eqref{eq:cliff-E_m-1} , we can compute
\begin{align*}
    &E_{m-1} \le \frac n{m-1} + \frac {n}{m-1} E_{m}, \\ 
    &\text{and}\\
    &E_{m}\le \alpha^d \frac n{m}  + \alpha^d \frac {n}{\alpha m} E_{m+1},
\end{align*}
which results in 
\begin{align*}
&E_1 \le \sum_{i=0}^{m-2} \frac n{(i+1)!} \left( \frac n \alpha \right)^{i} +  \frac{\alpha}{(m-1)!}\left( \frac n \alpha \right)^{m-1}E_{m},
\end{align*}
and
\begin{align*}
&E_1 \le \sum_{i=0}^{m-2} \frac n{(i+1)!} \left( \frac n \alpha \right)^{i} +  \frac{n\cdot \alpha^{d+1}}{m!}\left( \frac n \alpha \right)^{m-1}+  \frac{\alpha^{d+1}}{m!}\left( \frac n \alpha \right)^{m}E_{m+1}.
\end{align*}

For $i\in[m+1..\ceil{3n/\alpha}]$, using Equation~\eqref{eq:cliff-E_i}, we have
\begin{align*}
E_1 &\le \sum_{i=0}^{m-2} \frac n{(i+1)!} \left( \frac n \alpha \right)^{i} + \alpha^{d+1} \sum_{i=m-1}^{\ceil{3n/\alpha}-1}\frac{n}{(i+1)!}\left( \frac n \alpha \right)^{i}
+  \frac{\alpha^{d+1}}{(\ceil{3n/\alpha})!}\left( \frac n \alpha \right)^{\ceil{3n/\alpha}}E_{\ceil{3n/\alpha}+1} \\
&\le \alpha^{d+2} \sum_{i=0}^{\infty}\frac{1}{i!}\left( \frac n \alpha \right)^{i}
+  \alpha^{d+1}\left( \frac e 3 \right)^{\ceil{3n/\alpha}}E_{\ceil{3n/\alpha}+1} \le \alpha^{d+2}e^{n/\alpha} + o\left(\alpha^{d+1} \right)  \\
&= (1+o(1))\alpha^{d+2}e^{n/\alpha},
\end{align*}
where we have $E_{\ceil{3n/\alpha}+1}= O(n)$ since the search points in the distance~$\ceil{3n/\alpha}+1$ have positive drift (see Equation~\eqref{eq:driftD}).
\end{proof}

Finally, by working out the informal arguments for the overall proof idea given at the beginning of this subsection, we rigorously prove the optimization time for the part~\ref{thm:runtime-cliff:m<k} in Theorem~\ref{thm:runtime-cliff} as follows.

\begin{proofof}{Theorem~\ref{thm:runtime-cliff} part~\ref{thm:runtime-cliff:m<k}}
We first prove that 
\[ (1-o(1))\left(n\ln n+ E_1\right) \le \expect{T} \le \frac{\alpha}{\alpha+1}n(\ln n+1) + E_1(1+O(\alpha/n)).\]

Regarding the lower bound, using Theorem~\ref{thm:posdrift}, the running time to reach a search point with a negative drift (at distance $\ceil{\frac n{\alpha+1}}\ge m+1$ where $m>1$) is at least~$(1-o(1))n\ln n$ for $\alpha=\omega(\sqrt{n})$. Since the algorithm flips at most one bit each iteration, at a point of time, one search point at distance~$1$ is reached if the initial search point is not the global optimum with probability~$2^{-n}$. Therefore, for the lower bound, we have $\expect{T} \ge (1-o(1))\left(n\ln n+ E_1\right)$.

Regarding the upper bound, we have the following inequalities
\[\expect{T} \le   E_n^{\ceil{\frac n{\alpha+1}}}+ E_{\ceil{\frac n{\alpha+1}}}^{m} + E_m + E_{m-1}^{0}.\]

In the following paragraphs, we aim at estimating the terms in the last expression.
\begin{itemize}
\item[$E_{m-1}^0$:] Via Equation~\eqref{eq:cliff-E_i}, we use the recurrence relation 
\begin{align} \label{eq:E_i+1toE_i}
E_{i+1}=\frac{\alpha i}{n-i} E_i - \frac{\alpha n}{n-i} \le \frac{\alpha i}{n-i} E_i.
\end{align}
For $i\le n/(\alpha+1)$, we have 
\begin{align} \label{eq:E_i+1<E_i}
\frac{\alpha i}{n-i} \le \alpha \cdot \frac{n/(\alpha+1)}{n-(n/(\alpha+1))} = \frac{\alpha}{\alpha+1} \cdot \frac{n}{n(1-1/(\alpha+1)} =1,
\end{align}
resulting in $E_{i+1} \le E_i$. For $m=2$, $E_{m-1}^0=E_1$ is immediately obtained from the definitions and for $m \ge 3$, we have
\begin{align*}
    E_{m-1}^0=\sum_{i=1}^{m-1}E_i &\le E_1 + E_2 + E_3(m-3) \\
    & \le E_1 + O\left(\alpha/n\right)E_1+O\left(\alpha/n\right)E_2 (m-3) \\
    & \le E_1 + O\left(\alpha/n\right)E_1+O\left(\alpha^2/n^2\right)E_1 (m-3),
\end{align*}
where we used Equation~\eqref{eq:E_i+1toE_i} for $E_2$ and $E_3$ and for $i\ge 4$, $E_i\le E_3$.

Since $m< \frac{n}{\alpha+1}-1 = O(n/\alpha)$, the last expression is bounded from above by
\begin{align*}
    E_1 + O\left(\alpha/n\right)E_1+O\left(\alpha^2/n^2\right)E_1 \cdot  O(n/\alpha)=E_1(1+O(\alpha/n)).
\end{align*}
Therefore, we have $E_{m-1}^0=E_1(1+O(\alpha/n))$.

\item[$E_{m}$:] Via Equation~\eqref{eq:cliff-E_m-1}, we have
\[E_m=\frac{m-1}{n-m+1}E_{m-1}-\frac{n}{n-m+1} \le \frac{m-1}{n-m+1}E_{m-1}.\]
Thus, $E_m = \Theta(m/n)E_{m-1}$. Since $E_{m-1}\le E_{m-1}^0 = E_1(1+O(\alpha/n))$, we get
\begin{align} \label{eq:E_mtoE_1}
E_m= \Theta(m/n) E_1(1+O(\alpha/n)).
\end{align}

\item[$E_{\ceil{\frac n{\alpha+1}}}^{m}$:] We have $E_{i+1} \le E_i$ for $i\le n/(\alpha+1)$ similarly to the paragraph corresponding to $E_{m-1}^0$ (Equation~\eqref{eq:E_i+1<E_i}). We compute
\begin{align*}
    E_{\ceil{\frac n{\alpha+1}}}^{m} &=\sum_{i=m+1}^{\ceil{\frac n{\alpha+1}}}E_i \le  \ceil{\frac n{\alpha+1}} E_{m+1}.
\end{align*}

Using Equation~\eqref{eq:cliff-E_m} and~\eqref{eq:E_mtoE_1}, for $m=o(n)$, we have 
\[\ceil{\frac n{\alpha+1}}  E_{m+1}=  \ceil{\frac n{\alpha+1}}  \Theta( m/n) E_m = \ceil{\frac n{\alpha+1}}  \Theta( m^2/n^2) E_1(1+O(\alpha/n)).\] 
Since $m<\ceil{\frac n{\alpha+1}}=o(\sqrt{n})$, we have $E_{\ceil{\frac n{\alpha+1}}}^{m} = o\left(E_1\right)\left(1+O(\alpha/n)\right)$.

 \item[$E_n^{\ceil{\frac n{\alpha+1}}}$:] Since Equation~\eqref{eq:cliff-E_i} denoting $E_i = \frac{n}{i} + \frac{n-i}{\alpha i}E_{i+1}$ for $i\in[\ceil{\frac{n}{\alpha+1}}..n]$ is the same as the corresponding recursive equation for \onemax, the drift equation is also the same as Equation~\eqref{eq:driftD}, so $E_n^{\ceil{\frac n{\alpha+1}}}=\expect{T^\prime}$ can be estimated by Theorem~\ref{thm:climb-pos-drift}, where $T^\prime$ is the first time that the algorithm finds a solution with distance~$\ceil{\frac n{\alpha+1}}$, resulting in $E_n^{\ceil{\frac n{\alpha+1}}}\le \frac{\alpha}{\alpha+1}n(\ln n+1)$.
 \end{itemize}
 
 Altogether, we have 
 \[\expect{T}\le E_n^{\ceil{\frac n{\alpha+1}}}+ E_{\ceil{\frac n{\alpha+1}}}^{m} + E_m + E_{m-1}^{0} \le \frac{\alpha}{\alpha+1}n(\ln n+1) + E_1(1+O(\alpha/n)),\]
 and using Lemma~\ref{lem:E1:m<kstar}, we obtain
 \[(1-o(1))n\ln n +\left(\frac 1{\sqrt{2\pi}e^{\alpha/(12n)}}-o(1)\right) \frac{\alpha^{d+2}e^{n/\alpha}}{ \sqrt{n/\alpha}} \le \expect{T} \le \frac{\alpha}{\alpha+1}n(\ln n+1) + \alpha^{d+2}e^{n/\alpha} + o(n).\]
 Since $d\ge 1$ and $\alpha=\omega(\sqrt{n})$, $\alpha^{d+2} = \omega(n^{1.5})$, so we have
 \[\left(\frac 1{\sqrt{2\pi}e^{\alpha/(12n)}}-o(1)\right) \frac{\alpha^{d+2}e^{n/\alpha}}{ \sqrt{n/\alpha}} \le \expect{T} \le  (1-o(1))\alpha^{d+2}e^{n/\alpha}.\]
\end{proofof}

\subsection{Runtime of the \ooea on $\cliff_{d,m}$}
\label{sec:oneoneea-bounds}

In order to compare the Metropolis algorithm with evolutionary algorithms, we estimate the optimization time of the \ooea on \cliff functions. An expected runtime~of $\Theta(n^m)$ is has already been proven in~\cite{PaixaoHST17} for the classic case $d=m-3/2$ and mutation rate $p = \frac 1n$. 

In the following theorem, we prove an upper bound on the optimization time of the \ooea with general mutation rate~$p$ on $\cliff_{d,m}$. Since our main aim is showing that the \oea in many situations is faster than the MA, we prove no lower bounds. We note that for $k$ or $p$ not too large, one could show matching lower bounds with the methods developed in~\cite{DoerrLMN17,BamburyBD21}.

\begin{theorem} \label{thm:ea-cliff}
Consider the \ooea with general mutation rate $0 < p < \frac 12$ optimizing $\cliff_{d,m}$ with~arbitrary $m$ and $1 \le d < m-1$. Then the expected optimization time is at most 
\[
E[T] \le p^{-1} (1-p)^{-n+1} (1 + \ln n) + \binom{m}{\floor{d}+2}^{-1} p^{-\floor{d}-2} (1-p)^{-n+\floor{d}+2}.
\]
Any $p$ minimizing this bound satisfies $p \le \frac{\floor{d}+2}{n}$. If $m = O(n^{1/2} / \log n)$ and $p = \frac{\lambda}{n}$ for some $0 < \lambda \le \floor{d}+2$, then this bound is $E[T] \le (1+o(1)) \frac{e^\lambda}{\lambda^{\floor{d}+2}} \binom{m}{\floor{d}+2}^{-1}n^{\floor{d}+2}$. This latter bound is minimized for $\lambda = \floor{d}+2$, which yields 
\[
E[T] \le (1+o(1))  \binom{m}{\floor{d}+2}^{-1} \left(\frac{ne}{\floor{d}+2}\right)^{\floor{d}+2}.
\]
\end{theorem}

Since we analyze an elitist algorithm, we can use Wegener's~\cite{Wegener01} fitness level argument, which estimates the expected runtime by the sum of the expected times to leave each fitness level (apart from the optimal one). 
%The reason for considering separately the case that $d$ is an integer is that in this case the search points with exactly~$n-m+d+1$ one-bits have the same fitness as the local optima with $n-m$ one-bits, which renders the analysis of the time to leave this fitness level more complicated. 

In our proof, we will need the following elementary estimate.
\begin{lemma}\label{lem:asymp}
Let $m = O(n^{1/2} / \log n)$, $2 \le D \le m$, and $0 < p \le \frac{D}{n}$. Then 
\[p^{-1} (1-p)^{-n+1} \ln(n) = o\left(p^{-D} (1-p)^{-n+D} \binom{m}{D}^{-1}\right).
\]
\end{lemma}

\begin{proof}
  It suffices to show that $p^{-D+1} (1-p)^{D-1} \binom{m}{D}^{-1} = \omega(\log n)$. Assuming $n$ to be sufficiently large, we have $p \le \frac Dn \le \frac mn \le \frac 12$ and thus $p^{-D+1} (1-p)^{D-1} \binom{m}{D}^{-1} \ge (\frac {2D}{n})^{-D+1} (\frac D{em})^D = \frac{D}{em} (\frac{n}{2em})^{D-1} \ge \frac{2}{em} \frac{n}{2em} = \Omega((\log n)^2) = \omega(\log n)$.
\end{proof}

\begin{proofof}{Theorem~\ref{thm:ea-cliff}}
Let us first assume that $d \notin \N$ as in this case each set $L_i = \{x \in \{0,1\}^n \mid \|x\|_1=i\}$ is a separate fitness level of \cliff. Let $s_i$ be the probability that a single iteration of the \oea starting with a solution in $L_i$ ends with a solution of better fitness (note that this is independent of the particular solution from $L_i$). Then Wegener's~\cite{Wegener01} fitness level theorem gives the bound $E[T] \le \sum_{i=0}^{n-1} \frac 1 {s_i}$ for the runtime $T$ of the algorithm. 

For $i \neq n-m$, that is, a level different from the local optimum, we have $s_i \ge (n-i) p (1-p)^{n-1}$ simply by regarding the event that the mutation operator flips a single zero-bit. Consequently, $\sum_{i \neq n-m} \frac 1 {s_i} \le p^{-1} (1-p)^{-n+1} (1 + \ln n) =: T'$. 
% Intuitively, $T'$ is also an upper bound on \onemax 
% obtained via the fitness-level technique
% except that level $n-m$ is left out.

To estimate $s_{n-m}$, we note that if the current search point is  on the local optimum, then flipping any $\floor{d}+2$ of the zero-bits and no other bits leads to a better solution. Hence
\[
s_{n-m} \ge \binom{m}{\floor{d}+2}p^{\floor{d}+2}(1-p)^{n-\floor{d}-2}.
\] 

Consequently, by the fitness level argument,
\begin{align*}
    E[T] \le \sum_{i=0}^{n-1} \tfrac 1 {s_i} 
    & \le  T' + \binom{m}{\floor{d}+2}^{-1} p^{-\floor{d}-2} (1-p)^{-n+\floor{d}+2},
\end{align*}
which proves our claim for arbitrary mutation rate $p$. We note that the expression $p^x (1-p)^{n-x}$ is maximal exactly for $p = \frac xn$. Consequently, both $T'$ and our estimate for $\frac 1 {s_{n-m}}$ are strictly increasing for $p \ge \frac{\floor{d}+2}{n}$. Hence any mutation rate minimizing our estimate for the expected runtime cannot be larger than $\frac{\floor{d}+2}{n}$.  

If $p \le \frac{\floor{d}+2}{n}$, then by our assumption $m = O(n^{1/2} / \log n)$ and Lemma~\ref{lem:asymp}, we have $E[T] \le (1+o(1)) \binom{m}{\floor{d}+2}^{-1} p^{-\floor{d}-2} (1-p)^{-n+\floor{d}+2}$, which yields the remaining small claims.

When $d$ is an integer, then only the $L_i$ with $i \in  [0..n-1] \setminus \{n-m,n-m+d+1\} =: I$ form a complete fitness level of non-optimal solutions. The solutions on the remaining two Hamming levels have equal fitness. We estimate the probability to leave this level by the probability to generate a search point in $L_{n-m+d+2}$. By~\cite[Lemma~6.1]{Witt13}, since $p \le \frac 12$, this probability is at least the probability of finding an improvement from the (farther) level $L_{n-m}$. Hence $s^* = \binom{m}{d+2} p^{d+2} (1-p)^{n-d-2}$ is a lower bound for the probability to leave this fitness level, independent of the current search point. The resulting runtime estimate $E[T] \le \sum_{i \in I} \frac 1 {s_i} + \frac 1 {s^*}$ is identical to our above estimate for the $d$-value $d + 0.5$, which concludes this proof.
\end{proofof}
%%%%%%%%%%%%%%%%%

\section{Supplementary Plots to Experimental Section of the Paper}

Figure~\ref{fig:ma-cliff-different-alpha} shows how the runtime of the MA depends on the parameter~$\alpha$ under various choices 
of \cliff parameters. These experiments motivated us to concentrate on $\alpha\in[20,40]$.
Figures~\ref{fig:comparison-50}, \ref{fig:comparison-100}, \ref{fig:combi-ma-ea-60} and \ref{fig:combi-ma-ea-80} are results 
of experiments in the set-up of Section \emph{Experiments} of the main paper for additional parameter choices.

 \begin{figure} 
     \centering
 	\includegraphics[width=0.7\linewidth]{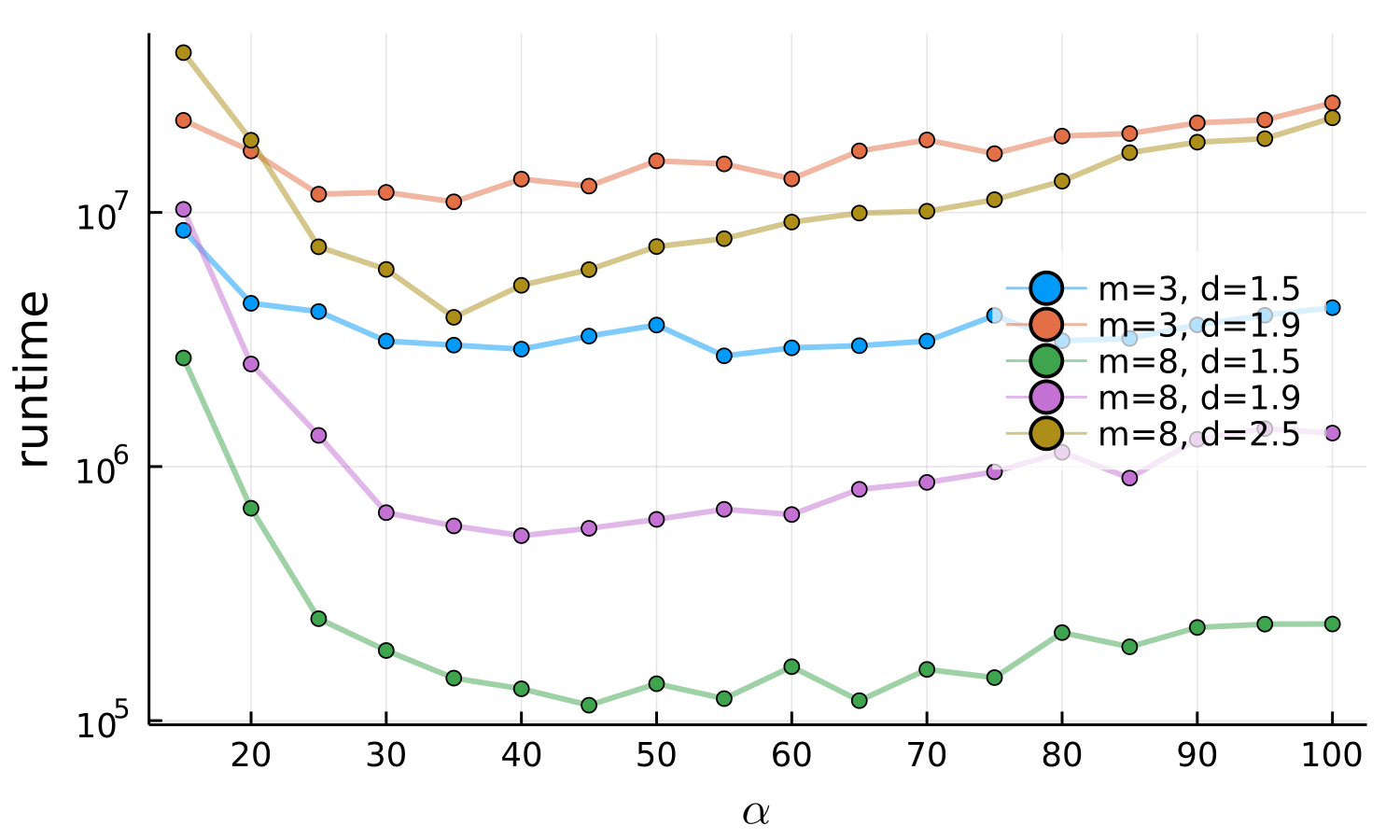}
 	\caption{Average number (over 100 runs) of iterations Metropolis with different cooling parameters~$\alpha$ took to optimize $\cliff_{d,m}$ functions of size~$n=100$ with different settings of $d$ and $m$.}
 	\label{fig:ma-cliff-different-alpha}
 \end{figure}

\begin{figure}[htbp] 
    \centering
	\includegraphics[width=0.7\linewidth]{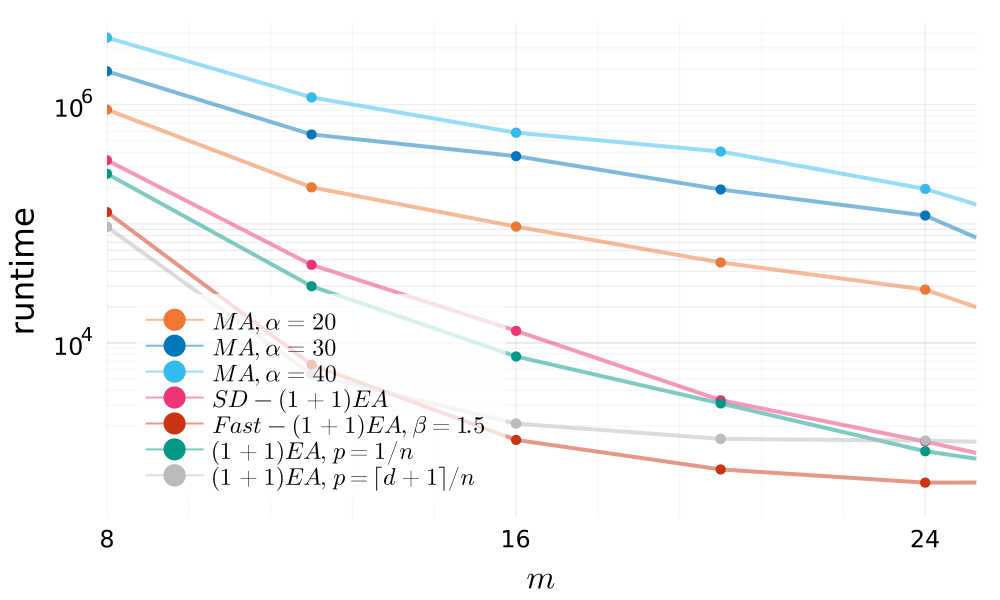}
	\caption{Comparison of MA with \oea and its variants on $\cliff_{m,d}$ for $n=50$, $d=3$ and $m\in \{8,12,\dots,32\}$; averaged over 50 runs.}
	\label{fig:comparison-50}
\end{figure}

\begin{figure}[htbp] 
    \centering
	\includegraphics[width=0.7\linewidth]{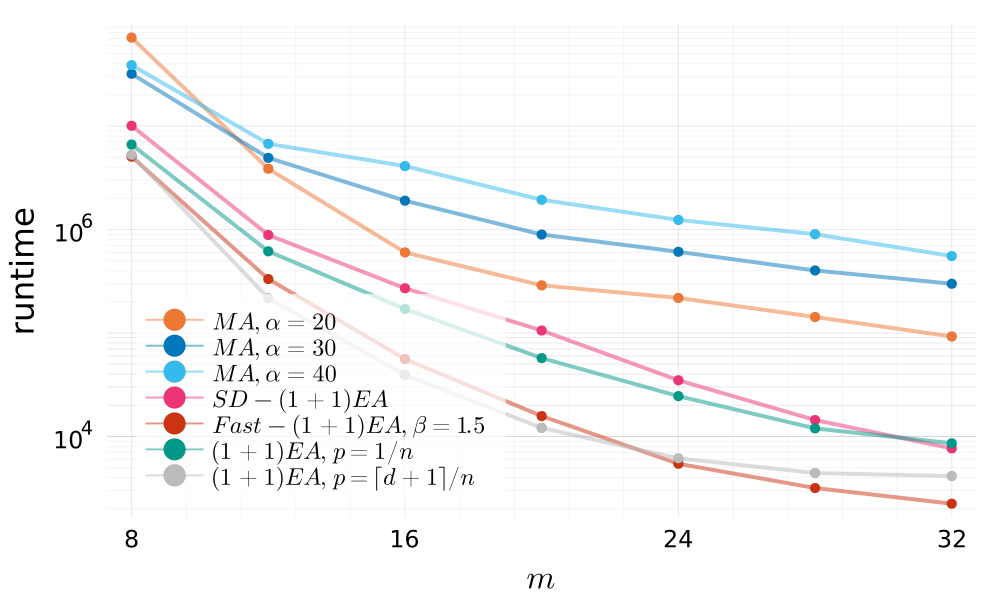}
	\caption{Comparison of MA with \oea and its variants on $\cliff_{m,d}$ for $n=100$, $d=3$ and $m\in \{8,12,\dots,32\}$; averaged over 50 runs.}
	\label{fig:comparison-100}
\end{figure}

\begin{figure}[htbp] 
    \centering
	\includegraphics[width=0.7\linewidth]{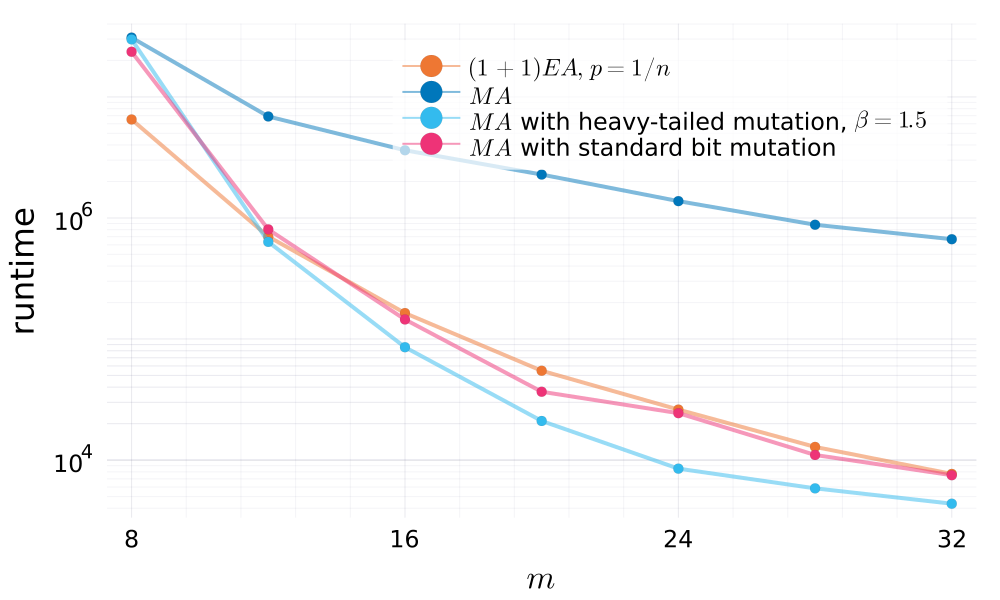}
	\caption{Comparison of \oea and different variants of MA, including  global mutation, on $\cliff_{m,d}$ for $n=100$, $\alpha=40$, $d=3$ and $m\in \{8,12,\dots,32\}$; averaged over 50 runs.}
	\label{fig:combi-ma-ea-40}
	\end{figure}

\begin{figure}[htbp] 
    \centering
	\includegraphics[width=0.7\linewidth]{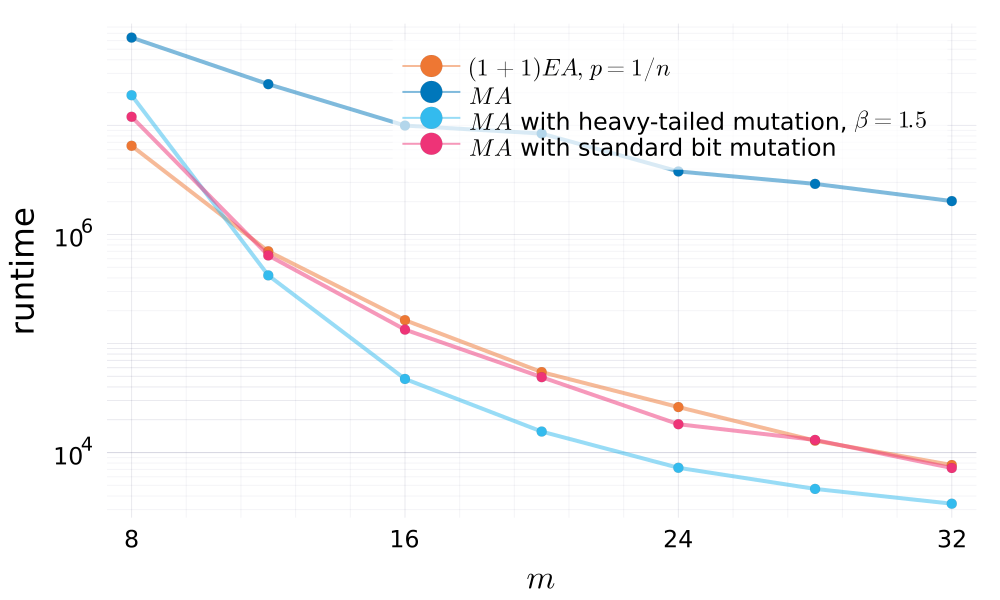}
	\caption{Comparison of \oea and different variants of MA, including  global mutation, on $\cliff_{m,d}$ for $n=100$, $\alpha=60$, $d=3$ and $m\in \{8,12,\dots,32\}$; averaged over 50 runs.}
	\label{fig:combi-ma-ea-60}
\end{figure}

\begin{figure}[htbp] 
    \centering
	\includegraphics[width=0.7\linewidth]{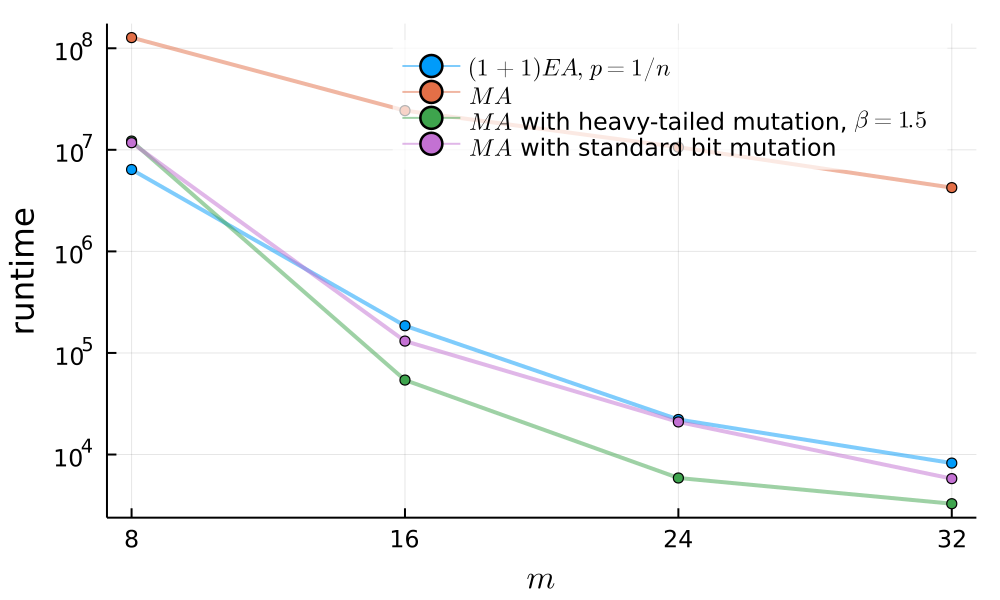}
	\caption{Comparison of \oea and different variants of MA, including  global mutation, on $\cliff_{m,d}$ for $n=100$, $\alpha=80$, $d=3$ and $m\in \{8,16,24,32\}$; averaged over 50 runs.}
	\label{fig:combi-ma-ea-80}
\end{figure}

% \clearpage
% \onecolumn
% \appendix

% \section{Omitted Proofs}
}{}%end ifthenelse{arxiv}

} %sloppy
\end{document}